\documentclass[3p,times]{elsarticle}

\usepackage{graphicx}\usepackage{multirow}\usepackage{amsmath,amssymb,amsfonts}\usepackage{amsthm}\usepackage{mathtools}\usepackage{xcolor}\usepackage{booktabs}\usepackage{algorithm}\usepackage{algorithmicx}\usepackage{algpseudocode}\usepackage{caption}\usepackage{subcaption}\usepackage{nicefrac}\usepackage{url}\usepackage[colorlinks=true,linkcolor=blue,citecolor=blue,urlcolor=blue]{hyperref}

\graphicspath{{images/}}
\DeclareGraphicsExtensions{.pdf,.png,.jpg}

\theoremstyle{plain}\newtheorem{theorem}{Theorem}\newtheorem{proposition}{Proposition}\newtheorem{lemma}{Lemma}\newtheorem{corollary}{Corollary}\newtheorem{assumption}{Assumption}

\theoremstyle{remark}\newtheorem{remark}{Remark}

\theoremstyle{definition}\newtheorem{definition}{Definition}

\newcommand{\coderepo}{\url{https://github.com/intsystems/nas-for-moe}}

\journal{Neurocomputing}

\begin{document}

\begin{frontmatter}

\title{Structure Aware Neural Architecture Search for Mixture of Experts}

\author{Petr Babkin\corref{cor1}}
\ead{p.k.babkin@gmail.com}
\author{Oleg Bakhteev}

\cortext[cor1]{Corresponding author.}

\begin{abstract}
Neural Architecture Search (NAS) has so far rarely been applied to Mixture-of-Experts (MoE) models, and existing MoE designs leave the alignment between experts and the structure of the data to emerge on its own. We propose an architecture search framework that makes this alignment an explicit search variable: the assignment of data clusters to experts is optimised jointly with the per-expert architectures. We cast the joint problem as a cluster-aware likelihood maximisation, show that it coincides with the incomplete-data maximum likelihood of a latent-variable mixture, and solve it by a generalised Expectation--Maximisation procedure whose otherwise intractable expert-quality term is supplied by an adaptively refined surrogate. We prove that the iterates converge whenever the surrogate errors are summable, and that at every limit point no candidate the search produces improves the true objective. On a heterogeneous image-classification mixture the method recovers the underlying domain partition on $95\%$ of clusters without ever observing domain labels, and on that benchmark and a four-domain time-series forecasting one alike it outperforms the MoE and NAS baselines that likewise use no label information.
\end{abstract}


\end{frontmatter}

\section{Introduction}
\label{sec:introduction}

The Mixture-of-Experts (MoE) architecture, in which a gating function routes each input to a small subset of specialized experts, has been well-established in the literature for decades \cite{jacobs1991adaptive, mu2025comprehensive} and is now a standard building block of large-scale deep models \cite{fedus2022switch, jiang2024mixtral}.
Neural Architecture Search (NAS) has in parallel matured into a wide field \cite{ren2021comprehensive, liu2018darts}, automating a design step that otherwise requires substantial domain expertise and manual tuning.
Despite the individual successes of both methodologies, the application of NAS techniques specifically to MoE architectures remains a relatively underexplored research area, as the survey of prior work in Section~\ref{sec:related_work} below shows, presenting significant opportunities for advancing both efficiency and performance in large-scale neural architectures.
In contrast to the classical approach where all experts share the same architecture \citep{fedus2022switch}, our study allows architectural heterogeneity between experts.
We claim that this flexibility enhances the overall efficiency of the resulting architecture by enabling more specialized and adaptable expert networks.

The MoE architecture is inherently well-suited for data with strong clustering properties, both in theory \cite{chen2022towards} and in practice on multimodal datasets \cite{chai2022mixture, guo2018multi, hu2025multi}.
However, despite these promising results, the underlying mechanisms that enable experts to align with distinct modalities remain insufficiently understood, and existing methods leave that alignment to emerge on its own rather than controlling it explicitly.

Given the foregoing, we formulate the following research question:
\textit{how can we design conditions under which expert architectures adapt to specific data clusters?}
In this study, we aim to address this question by examining how models evolve during adaptation to specific data modalities.

\vspace{2mm}

\textbf{Contributions}

\vspace{2mm}

\begin{itemize}
    \item \textbf{Problem formulation.} We propose a new formulation of architecture search for Mixture-of-Experts models as a cluster-aware likelihood-maximisation problem, and give a theoretical grounding for it by showing its equivalence to the incomplete-data maximum likelihood of a latent-variable mixture.
    \item \textbf{Algorithm.} We propose an algorithm that solves the stated problem, based on the Expectation--Maximisation (EM) algorithm with an adaptively refined surrogate model.
    \item \textbf{Convergence guarantee.} We prove a convergence theorem for the proposed algorithm.
    \item \textbf{Benchmarks and experiments.} We construct datasets on which mixture-of-experts approaches have a clear advantage, and conduct comparative experiments against several baselines on them.
\end{itemize}

The source code of the proposed method and of all experiments is
publicly available at \coderepo.

\subsection{Related Work}
\label{sec:related_work}

\noindent

\vspace{2mm}

\textbf{Mixture-of-Experts.}
The Mixture-of-Experts (MoE) framework was originally introduced
in~\cite{jacobs1991adaptive,jordan1994hierarchical,jacobs1997bayesian}
as a way to combine specialised statistical models through a learnable
gating mechanism. Subsequent theoretical analyses characterised
identifiability and convergence in classical MoE
formulations~\cite{nguyen2023general,nguyen2023demystifying,piwko2025divide,rasmussen2001infinite}.
With the rise of deep learning, MoE has been revived in the form of
sparsely activated neural networks, where the gate routes each input
to only a small subset of
experts~\cite{eigen2013learning,shazeer2017outrageously,riquelme2021scaling,fedus2022switch}.
This sparse formulation enables training of models with hundreds of
billions of parameters while keeping per-token compute roughly
constant, and is now a standard building block of modern large-scale
transformers~\cite{jiang2024mixtral,dai2024deepseekmoe,qiu2025demons}.

\vspace{2mm}

\textbf{Domain specialisation of experts.}
A complementary line of work constructs experts that specialise to
data domains \emph{explicitly}. DEMix layers~\cite{gururangan2022demix}
condition expert selection on document-level domain metadata;
Branch-Train-Merge~\cite{li2022branch} trains independent expert
language models per domain and observes that ensembles built from
random data splits perform poorly; c-BTM~\cite{gururangan2023scaling}
removes the need for metadata by discovering domains via $k$-means
clustering of corpus embeddings and dedicating one expert to each
cluster; Branch-Train-MiX~\cite{sukhbaatar2024branch} subsequently
merges such domain experts into a single sparse MoE layer with a
learned router. In all of these methods the cluster-to-expert mapping
is fixed a priori and all experts share the same architecture, whereas
we treat both the mapping and the per-expert architectures as
optimisation variables. The need for such an explicit mechanism is
supported by analyses of trained sparse MoE models: the routers of
Mixtral exhibit no topic-level expert
specialisation~\cite{jiang2024mixtral}, and OLMoE finds that domain
specialisation emerges only partially --- and in an uncontrolled
manner --- when training from scratch~\cite{muennighoff2025olmoe}.
Sparse MoE has also been adopted by time-series foundation models such
as Time-MoE~\cite{shi2025timemoe} and
Moirai-MoE~\cite{liu2025moiraimoe}; the latter routes tokens through
$k$-means centroids of pretrained representations, which is close in
spirit to our cluster-level formulation, but neither model searches
expert architectures or optimises the cluster-level assignment itself.

\vspace{2mm}

\textbf{Neural Architecture Search.}
Neural Architecture Search (NAS) automates the design of neural
networks and has matured into a wide research
field~\cite{ren2021comprehensive}. Many-shot approaches based on
reinforcement learning~\cite{zoph2016neural,baker2016designing,zoph2018learning}
or evolutionary search~\cite{real2019regularized,yang2020cars}
explicitly train candidate architectures and select the best
performer. One-shot methods reduce the search cost by training a
single supernet that contains all candidate architectures as
sub-networks~\cite{liu2018darts,chen2020stabilizing,you2020greedynas,nayman2019xnas}.
DARTS~\cite{liu2018darts} is the most prominent representative of this
family and forms one of the baselines we compare against in our
experimental study.

\vspace{2mm}

\textbf{Neural Architecture Search for Mixture-of-Experts.}
The intersection of NAS and MoE is a comparatively recent and still
sparsely populated research area. AutoMoE~\cite{jawahar2022automoe}
applies an evolutionary search on top of a shared supernet to obtain
heterogeneous MoE architectures for neural machine translation;
Brainformers~\cite{zhou2023brainformers} search over block types
including sparse-FFN layers at foundation-model scale; and a recent
automated pipeline search explores heterogeneous 4-expert
mixtures~\cite{lukhi2026heterogeneous}.
CMN~\cite{han2024cmn} co-designs MoE topologies for
computing-in-memory hardware. These methods operate at a scale and in
domains (large-scale NMT and language modelling on shared supernets)
that are not directly comparable to our regime; we therefore cite them
as the closest NAS-for-MoE prior work but do not reproduce them, and
instead include an AutoMoE-style baseline that searches heterogeneous
per-expert architectures under a conventional gate --- isolating the
contribution of jointly optimising the cluster-to-expert assignment,
which none of the above treat as a search variable. The work most closely related to ours
is MoE-NAS~\cite{mecharbat2025moenas}, which optimises the number of
experts per layer but keeps the expert architectures shared. In
contrast, we explicitly search over per-expert architectures and tie
this search to the cluster structure of the input data.

\vspace{2mm}

\textbf{Theoretical Understanding.}
A growing body of theoretical work studies why MoE models perform
well on data with strong cluster or modality structure.
\cite{chen2022towards} proves on simplified settings that MoE
provably benefits from clusterable data, while~\cite{li2024theory}
analyses the optimisation landscape of sparse gates. Empirical
evidence from multimodal and multi-task
learning~\cite{chai2022mixture,guo2018multi,hu2025multi,zhou2022mixture}
further supports the claim that experts naturally specialise to
distinct sub-distributions. Our work builds on this perspective and
turns it into an explicit search objective by optimising over
\emph{per-cluster} expert architectures rather than treating the
cluster structure as a by-product of training.

\vspace{2mm}
 
\section{Problem Statement}
\label{sec:problem_statement}

This section formalises what exactly is being searched for. We first
define the search space --- the discrete space of per-expert
architectures together with the cluster-level routing that jointly
specify an MoE model (Section~\ref{subsec:searchspace}) --- and then
cast the search over this space as a cluster-aware
likelihood-maximisation problem (Section~\ref{subsec:objective}).

\subsection{Search Space of MoE}
\label{subsec:searchspace}

We adopt the \emph{cell-based} search space that is standard in the NAS
literature~\cite{zoph2018learning, liu2018darts}: instead of searching
over entire networks, one searches over a single small building block
--- a \emph{cell} --- and a network is obtained by stacking several
copies of that block inside a fixed skeleton (a stem, a reduction stage
and a task-specific head). Each expert is one such network, so searching
for an expert amounts to searching for its cell.

A cell is a directed acyclic graph whose nodes are computation steps. It
has a fixed number
of nodes and each node applies a single operation, drawn from a fixed
finite set $\mathcal{O}$ of primitives appropriate to the data modality
(the concrete sets used in our experiments are given in
Section~\ref{sec:experiments}), to the output of one preceding node;
the first node reads the cell input. An architecture is therefore
determined by two discrete choices per node: \emph{which} operation the
node applies, and \emph{from which} earlier node it takes its input. We
gather all these choices into a single architecture vector
$\boldsymbol{\alpha}$. It describes one cell and, through the fixed
skeleton, the whole expert network: all copies of the cell within an
expert share the same $\boldsymbol{\alpha}$, while their weights are
trained independently.
Since both
$\mathcal{O}$ and the number of nodes are finite, the search space
$\mathcal{A}$ of all admissible $\boldsymbol{\alpha}$ is finite, and
neural architecture search reduces to selecting the best
$\boldsymbol{\alpha} \in \mathcal{A}$. For an architecture
$\boldsymbol{\alpha}_k$ with trained weights $\boldsymbol{\theta}_k$ we
denote the resulting network by
$f(\boldsymbol{\alpha}_k, \boldsymbol{\theta}_k): \mathcal{X} \to \mathcal{Y}$.

A Mixture-of-Experts (MoE) model combines $K$ such experts, each meant
to specialise on a different region of the input space. Stacking the
expert architectures into
$\boldsymbol{\alpha} = [\boldsymbol{\alpha}_1^{\!\top}, \ldots, \boldsymbol{\alpha}_K^{\!\top}]^{\!\top}$,
the MoE prediction is the routing-weighted combination of expert outputs
\begin{equation}
\label{eq:moe_pred}
g(\boldsymbol{\alpha}, \boldsymbol{\theta}, \boldsymbol{w}_r)(\boldsymbol{x})
= \sum_{k=1}^{K} r_k(\boldsymbol{x}, \boldsymbol{w}_r)\, f(\boldsymbol{\alpha}_k, \boldsymbol{\theta}_k)(\boldsymbol{x}),
\end{equation}
where the gate values $r_k(\boldsymbol{x}, \boldsymbol{w}_r) \geq 0$ with
$\sum_{k=1}^{K} r_k(\boldsymbol{x}, \boldsymbol{w}_r) = 1$ are produced by
a routing network with parameters $\boldsymbol{w}_r$.

\textbf{Data clustering and routing.}
Prior to architecture search, the dataset is partitioned into $M$ clusters,
$\mathcal{D} = \mathcal{C}_1 \cup \mathcal{C}_2 \cup \cdots \cup \mathcal{C}_M$,
each grouping samples of similar structure (e.g.\ a semantic class group or a data modality). Throughout the paper, $n \in \{1, \ldots, N\}$ indexes individual dataset samples, $m \in \{1, \ldots, M\}$ indexes clusters, and $k \in \{1, \ldots, K\}$ indexes experts. We write $\mathcal{C}_m \subseteq \{1, \ldots, N\}$ for the sample indices that belong to cluster $m$, and denote by $m(n) \in \{1, \ldots, M\}$ the cluster index of sample $n$.

Routing is defined at the cluster level, and it is \emph{stochastic}: every cluster $m$ carries a latent assignment variable $c_m \in \{1, \ldots, K\}$ naming the expert it is routed to, and
\begin{equation}
\label{eq:routing_def}
p(c_m = k) \;=\; r_{mk} \in [0, 1], \qquad \sum_{k=1}^{K} r_{mk} = 1,
\end{equation}
the probabilities being collected in a routing matrix $\mathbf{R} \in [0, 1]^{M \times K}$. The column $\mathbf{r}_k = \mathbf{R}_{:,k} \in [0, 1]^M$ encodes the clusters expert $k$ is responsible for. The latent variables $c_{1:M}$ are what makes the model below a genuine mixture; the search optimises the distribution $\mathbf{R}$ rather than any fixed assignment, and a hard cluster$\to$expert assignment is read off only at the end of optimisation, as $I_k = \{m : \arg\max_{k'} r_{mk'} = k\} \subseteq \{1, \ldots, M\}$, the set of cluster indices routed to expert $k$.

Unlike standard MoE training, where only the gating parameters $\boldsymbol{w}_r$ are learned for a fixed architecture, here \emph{finding the routing $\mathbf{R}$ is itself part of the architecture search}: the search jointly determines the per-expert architectures $\boldsymbol{\alpha}_{1:K}$ and the assignment of clusters to experts. This couples architecture selection with data specialization, since the best architecture for an expert depends on the cluster subset it is trained on.

\subsection{MoE Objective}
\label{subsec:objective}

We cast MoE architecture search as a likelihood-maximization problem. The goal is to choose, for every expert $k$, an architecture $\boldsymbol{\alpha}_k$ and its trained parameters $\boldsymbol{\theta}_k$, together with the routing distribution $\mathbf{R}$, so as to maximize the data likelihood under the mixture induced by the latent assignments $c_{1:M}$:
\begin{equation}
\label{eq:moe_likelihood}
\prod_{n=1}^{N} \sum_{k=1}^{K} p\bigl(y_n,\, c_{m(n)} = k \,\big|\, x_n, \boldsymbol{\alpha}_k, \boldsymbol{\theta}_k\bigr)
\;\to\; \max_{\boldsymbol{\alpha}_{1:K},\, \boldsymbol{\theta}_{1:K},\, \mathbf{R}},
\end{equation}
where the event $c_{m(n)} = k$ reads ``sample $n$'s cluster is routed to expert $k$''. Summing over $k$ marginalises this event out, which is what makes~\eqref{eq:moe_likelihood} a mixture: had the cluster$\to$expert map been a fixed partition instead of a distribution, the membership indicator would be deterministic, exactly one term of the inner sum would survive, and the objective would collapse to the likelihood of a hard assignment.

Each joint term factorizes into a routing probability and a conditional likelihood:
\begin{equation}
\label{eq:moe_factor}
p\bigl(y_n,\, c_{m(n)} = k \,\big|\, x_n, \boldsymbol{\alpha}_k, \boldsymbol{\theta}_k\bigr)
= p\bigl(c_{m(n)} = k\bigr)\, p\bigl(y_n \,\big|\, x_n, \boldsymbol{\alpha}_k, \boldsymbol{\theta}_k,\, c_{m(n)} = k\bigr).
\end{equation}
The membership probability is the cluster-level routing,
$p\bigl(c_{m(n)} = k\bigr) = r_{m(n), k}$, by the definition~\eqref{eq:routing_def} of $\mathbf{R}$. The conditional factor is the expensive one: evaluating it exactly for a candidate architecture would mean actually training that architecture --- and, in the MoE setting, all $K$ of them --- on the cluster subset the routing assigns to it. We therefore replace exact training by a learned \emph{surrogate}, and we let that surrogate be \emph{vector-valued}, with one component per cluster, because two indices must be kept apart here: the factor is indexed by the sample $n$, hence by its cluster $m(n)$, and asks how well expert $k$ --- trained on the clusters that $\mathbf{r}_k$ assigns to it --- explains the data of cluster $m(n)$.

\begin{definition}
\label{def:surrogate}
Let the training data be partitioned into $M$ clusters
$\mathcal{D}_{\text{train}} = \mathcal{C}_1 \cup \cdots \cup \mathcal{C}_M$;
for $\boldsymbol{b} \in \{0, 1\}^M$ let
$\mathcal{D}_{\boldsymbol{b}} = \bigcup_{m: b_m = 1} \mathcal{C}_m$ be
the cluster subset selected by~$\boldsymbol{b}$ and
$\boldsymbol{\theta}^{*}_{\boldsymbol{\alpha}, \boldsymbol{b}}$ the
parameters obtained by training $\boldsymbol{\alpha}$ on it --- for a
binary routing column we abbreviate these to $\mathcal{D}_{\mathbf{r}_k}$
and $\boldsymbol{\theta}_k^{*}$. A \emph{vector-valued surrogate
function}
\begin{equation}
\label{eq:surrogate_vertices}
\begin{gathered}
\boldsymbol{u}: \mathcal{A} \times [0, 1]^M \to \mathbb{R}_{\geq 0}^{M},
\qquad \boldsymbol{b} \in \{0, 1\}^M,\\[2pt]
u_m(\boldsymbol{\alpha}, \boldsymbol{b}) \;\approx\;
\mathrm{NLL}^{(m)}_{\mathrm{val}}\!\left(f(\boldsymbol{\alpha}, \boldsymbol{\theta}^*_{\boldsymbol{\alpha}, \boldsymbol{b}})\right)
\;=\; - \frac{1}{|\mathcal{C}_m^{\text{val}}|}
\sum_{n \in \mathcal{C}_m^{\text{val}}}
\log p\bigl(y_n \,\big|\, x_n, \boldsymbol{\alpha}, \boldsymbol{\theta}^{*}_{\boldsymbol{\alpha}, \boldsymbol{b}}\bigr),
\end{gathered}
\end{equation}
has as its $m$-th component the \emph{mean per-sample negative
log-likelihood} (NLL) of architecture $\boldsymbol{\alpha}$, trained on
$\mathcal{D}_{\boldsymbol{b}}$, when evaluated on the validation part
$\mathcal{C}_m^{\text{val}}$ of cluster $m$ --- all without performing
the actual training. The NLL is the cross-entropy of the correct class
in the classification experiments of Sections~\ref{subsec:toy}
and~\ref{subsec:cifar}, and the negative log-density of the predictive
distribution in the forecasting experiment of Section~\ref{subsec:ts}.
Components with $b_m = 0$ are \emph{out-of-region} queries: they predict
how the expert behaves on clusters it was never trained on, which is
exactly the evidence the routing search needs. The two arguments thus
play distinct roles: the mask (in the search, the routing column
$\mathbf{r}_k$) selects the \emph{training} region, $m$ selects the
\emph{evaluation} cluster.

Only the vertices $\boldsymbol{b}$ carry this operational meaning, since
only a binary mask corresponds to a real training run; the second
argument is nevertheless allowed to range over the whole cube, because
the routing column $\mathbf{r}_k$ that the search manipulates is in
general fractional. We fix the fractional values \emph{canonically}, as
the average over masks drawn independently per cluster,
\begin{equation}
\label{eq:mask_extension}
u_m(\boldsymbol{\alpha}, \mathbf{r}) \;=\;
\mathbb{E}_{\boldsymbol{b} \sim \mathrm{Bernoulli}(\mathbf{r})}
\bigl[u_m(\boldsymbol{\alpha}, \boldsymbol{b})\bigr]
\;=\!\!\sum_{\boldsymbol{b} \in \{0,1\}^M}\!\!
\Bigl(\prod_{m'=1}^{M} r_{m'}^{\,b_{m'}}\,(1 - r_{m'})^{1 - b_{m'}}\Bigr)\,
u_m(\boldsymbol{\alpha}, \boldsymbol{b}),
\end{equation}
which is precisely the quantity that the straight-through
Gumbel-Softmax M-step of Section~\ref{subsec:sgem} estimates when it
draws a binary mask from $\mathbf{r}$. The extension agrees
with~\eqref{eq:surrogate_vertices} at the vertices, is multilinear in
$\mathbf{r}$ and hence continuous, and inherits any bound satisfied on
the vertices --- the two properties the convergence analysis of
Section~\ref{subsec:theory} asks for
(Assumptions~\ref{ass:bounded} and~\ref{ass:cont}). A concrete
construction of $\boldsymbol{u}$ as a learned regressor with $M$ outputs
is given in Section~\ref{subsec:surr}.
\end{definition}

For the conditional factor we then adopt the standard MLE proxy: the experts are trained by minimising exactly this loss, so the mean per-sample NLL an expert attains on a group of samples is what training drives down. The surrogate being constant within a cluster, we read that factor off it: for every $n$ in cluster $m$,
\begin{equation}
\label{eq:moe_persample_lik}
p\bigl(y_n \,\big|\, x_n, \boldsymbol{\alpha}_k, \boldsymbol{\theta}_k,\, c_{m(n)} = k\bigr)
\;\approx\; e^{-u_m(\boldsymbol{\alpha}_k, \mathbf{r}_k)},
\end{equation}
i.e.\ $e^{-u_m}$ is the per-sample geometric mean of the conditional likelihood of expert $k$ over cluster $m$. As in the standard EM treatment of mixtures, the component is evaluated at its \emph{current} parameters: $u_m(\boldsymbol{\alpha}_k, \mathbf{r}_k)$ uses $\boldsymbol{\theta}_k^{*}$ trained on $\mathcal{D}_{\mathbf{r}_k}$, without retraining expert $k$ on $\mathcal{D}_{\mathbf{r}_k} \cup \mathcal{C}_m$. Substituting~\eqref{eq:moe_factor} and~\eqref{eq:moe_persample_lik} into~\eqref{eq:moe_likelihood} yields the surrogate-based objective
\begin{equation}
\label{eq:moe_surr_likelihood}
\prod_{n=1}^{N} \sum_{k=1}^{K} r_{m(n), k}\, e^{-u_{m(n)}(\boldsymbol{\alpha}_k, \mathbf{r}_k)}
\;\to\; \max_{\boldsymbol{\alpha}_{1:K},\, \mathbf{R}}.
\end{equation}

Because routing is decided at the cluster level, $r_{m(n), k} = r_{m, k}$ for every $n \in \mathcal{C}_m$, and by~\eqref{eq:moe_persample_lik} the conditional likelihood is likewise identical across the samples of a cluster. Taking the logarithm of~\eqref{eq:moe_surr_likelihood} and grouping the $N$ per-sample terms by cluster collapses each cluster's $|\mathcal{C}_m|$ identical log-factors into a single weighted term, yielding the \emph{cluster-aware MoE objective}:
\begin{equation}
\label{eq:moe_objective}
\mathcal{J}(\boldsymbol{\alpha}_{1:K}, \mathbf{R}) \;=\;
\sum_{m=1}^{M} |\mathcal{C}_m|\,\log \sum_{k=1}^{K} r_{mk}\, e^{-u_m(\boldsymbol{\alpha}_k, \mathbf{r}_k)}
\;\to\; \max_{\boldsymbol{\alpha}_{1:K},\, \mathbf{R}}.
\end{equation}
Intuitively, $r_{mk}$ specifies how strongly cluster $m$ is routed to expert $k$, $e^{-u_m(\boldsymbol{\alpha}_k, \mathbf{r}_k)}$ is the per-sample likelihood that expert $k$ attains \emph{on cluster $m$}, and the cluster-size weight $|\mathcal{C}_m|$ reflects that a cluster's contribution to the joint log-likelihood scales with the number of samples it contains. Problem~\eqref{eq:moe_objective} is solved by the algorithm of Section~\ref{subsec:sgem} via Generalised EM with an adaptively refined surrogate.
 
\section{Solving the Cluster-aware Objective by Surrogate-assisted EM}
\label{sec:main_part}

The objective~\eqref{eq:moe_objective} cannot be optimised directly:
every evaluation of a candidate architecture would require training it
from scratch on the cluster subset the routing assigns to it. We
therefore first construct a learned surrogate of the conditional
likelihood (Section~\ref{subsec:surr}), then embed it into a
generalized EM procedure that alternates between updating the routing
and the expert architectures (Section~\ref{subsec:sgem}), and finally
prove that the procedure converges even though the surrogate itself is
refined during the search (Section~\ref{subsec:theory}).

\subsection{Learning the Surrogate}
\label{subsec:surr}

We now give the concrete construction of the vector-valued surrogate
$\boldsymbol{u}$ of Definition~\ref{def:surrogate}.
Figure~\ref{fig:method_overview} gives an overview of the loop it sits
in, in which the surrogate scores candidate architectures against the
current cluster assignment and only a few selected candidates are ever
trained.

\begin{figure}[h]
	\centering
    \includegraphics[width=0.7\textwidth]{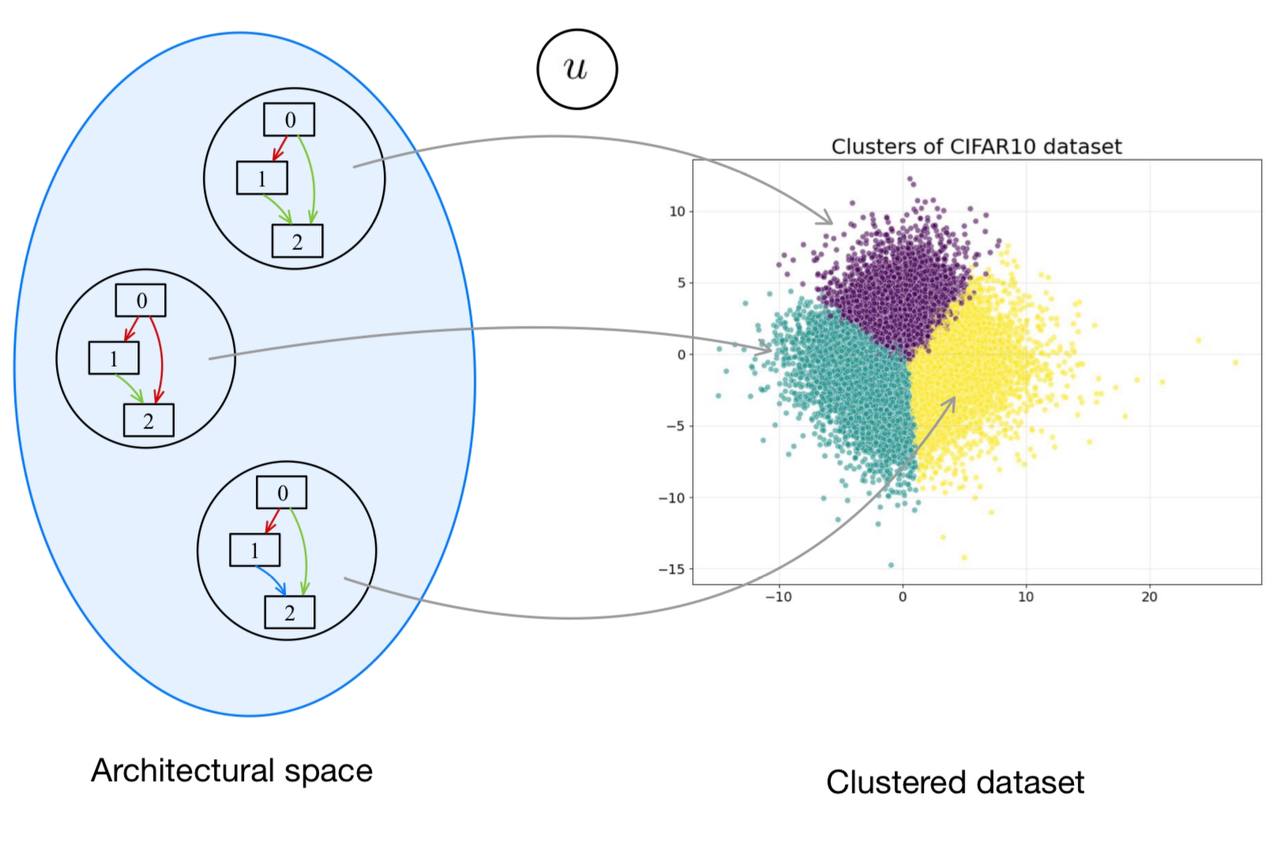}
	\caption{Method workflow: a surrogate model $u$ is used to select
    the most suitable expert architectures for each cluster.}
	\label{fig:method_overview}
\end{figure}

The surrogate is trained on a dataset of architecture/subset pairs
collected by genuine training:
\begin{equation}
\mathcal{D}_{\text{surr}} = \{(\boldsymbol{\alpha}^{(i)}, \boldsymbol{b}^{(i)}, \boldsymbol{y}^{(i)})\}_{i=1}^{S}, \qquad
y^{(i)}_m = \mathrm{NLL}^{(m)}_{\mathrm{val}}\!\bigl(f(\boldsymbol{\alpha}^{(i)}, \boldsymbol{\theta}^*_{\boldsymbol{\alpha}^{(i)}, \boldsymbol{b}^{(i)}})\bigr).
\end{equation}
A single training run yields the whole target vector at once --- one
forward pass over the validation set, with the per-sample losses bucketed
by cluster --- so the vector-valued surrogate costs no extra training
while providing $M$ regression targets per observation instead of one,
including targets for clusters outside $\boldsymbol{b}^{(i)}$.
We parametrise $\boldsymbol{u}_{\boldsymbol{\theta}}$ as a Graph Neural Network
encoder over the architecture DAG concatenated with an embedding of
the cluster mask $\boldsymbol{b}$, followed by an $M$-dimensional output
head, and minimise the mean squared regression loss
\begin{equation}
\label{eq:surr_training}
\min_{\boldsymbol{\theta}} \;
\frac{1}{S\,M}
\sum_{(\boldsymbol{\alpha}, \boldsymbol{b}, \boldsymbol{y}) \in \mathcal{D}_{\text{surr}}}
\sum_{m=1}^{M}
\bigl(u_{\boldsymbol{\theta}, m}(\boldsymbol{\alpha}, \boldsymbol{b}) - y_m\bigr)^2 .
\end{equation}
Targets are clipped at an upper quantile before the fit, which both keeps
the regression from being dominated by far-out-of-region outliers and
enforces the boundedness required by Assumption~\ref{ass:bounded}.

The surrogate is refined inside an active learning loop. Ranking
architecture candidates needs a single number, so the predicted vector
is reduced to a scalar --- either the size-weighted average over the
clusters the candidate was trained on,
\begin{equation}
\label{eq:scalar_u}
\bar u(\boldsymbol{\alpha}_k, \mathbf{r}_k) \;=\;
\Bigl(\textstyle\sum_{m \in I_k} |\mathcal{C}^{\text{val}}_m|\Bigr)^{-1}
\textstyle\sum_{m \in I_k} |\mathcal{C}^{\text{val}}_m|\; u_m(\boldsymbol{\alpha}_k, \mathbf{r}_k),
\end{equation}
which discards the dependence on the evaluation cluster, or, inside the
M-step, the responsibility-weighted loss of~\eqref{eq:mstep_alpha}.
At every iteration we use Monte-Carlo dropout to obtain a predictive mean
$\mu(\boldsymbol{\alpha}, \boldsymbol{b})$ and standard deviation
$\sigma(\boldsymbol{\alpha}, \boldsymbol{b})$ of that scalar over $T$ stochastic
forward passes, and score candidates by the Lower Confidence Bound
$\mathrm{LCB} = \mu - \sigma$ (since $\boldsymbol{u}$ is a loss, the most
promising candidates are those with the \emph{lowest} predicted NLL
together with high epistemic uncertainty). The most promising
candidate is trained from scratch, the resulting triple is appended
to $\mathcal{D}_{\text{surr}}$, and the surrogate is refit. Every
observation, and hence the whole fit~\eqref{eq:surr_training}, lives at a
vertex $\boldsymbol{b} \in \{0,1\}^M$; fractional masks are covered by
the extension~\eqref{eq:mask_extension}.

\subsection{Surrogate-assisted Generalized EM (SGEM)}
\label{subsec:sgem}

Treating the cluster$\to$expert routing as a latent variable turns the
surrogate objective~\eqref{eq:moe_objective} into a mixture model on
which Expectation--Maximisation applies directly, with the surrogate
$\boldsymbol{u}^{(t)}$ standing in for the otherwise-intractable expert quality. The
derivation is the standard EM treatment of a finite mixture, the
per-component data-fit term being the surrogate factor
$e^{-u_{m}(\boldsymbol{\alpha}_k, \mathbf{r}_k)}$, evaluated at the cluster
$m$ of the sample being explained.

\textbf{Latent-variable form of the objective.}
For every sample $n$ introduce a one-hot latent vector
$\mathbf{z}_n \in \{0, 1\}^K$ with $z_{nk} = 1 \iff$ sample $n$ is
generated by expert $k$ --- the per-sample copy of the cluster-level
assignment $c_{m(n)}$ of Section~\ref{subsec:searchspace}, so that
$z_{nk} = 1 \iff c_{m(n)} = k$ --- and collect them into the latent matrix
$\mathbf{Z} = \|z_{nk}\| \in \{0,1\}^{N \times K}$. The generative
story is the standard mixture one: the expert of sample $n$ is drawn
from its cluster's routing distribution,
$p(z_{nk} = 1) = r_{m(n), k}$, and the label is then emitted with that
expert's per-sample likelihood on the sample's own cluster,
$e^{-u_{m(n)}(\boldsymbol{\alpha}_k, \mathbf{r}_k)}$
[cf.~\eqref{eq:moe_persample_lik}]. The resulting \emph{complete-data}
likelihood is
\begin{equation}
\label{eq:complete_data}
p\bigl(\mathbf{y}, \mathbf{Z} \mid \boldsymbol{\alpha}_{1:K}, \mathbf{R}\bigr)
\;=\; \prod_{n=1}^{N} \prod_{k=1}^{K}
\Bigl( r_{m(n), k}\, e^{-u_{m(n)}(\boldsymbol{\alpha}_k, \mathbf{r}_k)} \Bigr)^{z_{nk}},
\end{equation}
in which the routing $r_{mk}$ acts as the mixing weight and
$e^{-u_{m(n)}(\boldsymbol{\alpha}_k, \mathbf{r}_k)}$ as the per-component
likelihood --- a component density that, as in any mixture model, is
evaluated at the observation it is asked to explain.

\emph{Equivalence with~\eqref{eq:moe_objective}.} Marginalising the
one-hot rows $\mathbf{z}_n$ out of~\eqref{eq:complete_data} and grouping
the samples by cluster recovers
$\exp \mathcal{J}(\boldsymbol{\alpha}_{1:K}, \mathbf{R})$, exactly as in
the passage from~\eqref{eq:moe_surr_likelihood}
to~\eqref{eq:moe_objective}. Maximising~\eqref{eq:moe_objective} is thus
the incomplete-data MLE of the latent-variable
model~\eqref{eq:complete_data}, which we now solve by alternating
maximisation of the variational free energy
$\mathcal{F}_u(q, \theta) = \mathbb{E}_q[\log p(\mathbf{y}, \mathbf{Z}\mid\theta)] + \mathbb{H}[q]$
over the posterior $q(\mathbf{Z})$ (E-step) and the parameters
$\theta = (\boldsymbol{\alpha}_{1:K}, \mathbf{R})$ (M-step).

\textbf{E-step.} At fixed $(\boldsymbol{\alpha}_{1:K}, \mathbf{R})$ the
free energy is maximised by the exact posterior
$q(\mathbf{Z}) = p(\mathbf{Z}\mid \mathbf{y}, \boldsymbol{\alpha}_{1:K}, \mathbf{R})$.
From~\eqref{eq:complete_data},
$\log q(\mathbf{Z}) \propto \sum_{n}\sum_{k} z_{nk}\bigl(\log r_{m(n),k}
- u_{m(n)}(\boldsymbol{\alpha}_k, \mathbf{r}_k)\bigr)$, so the posterior
factorises over samples and the responsibility of expert $k$ for sample
$n$ is the normalised mixture weight
$\gamma_{nk} \propto r_{m(n),k}\, e^{-u_{m(n)}(\boldsymbol{\alpha}_k, \mathbf{r}_k)}$.
This depends on $n$ only through its cluster $m(n)$, so all samples of a
cluster share one responsibility, which we write as
\begin{equation}
\label{eq:estep}
q_{mk} \;=\; \frac{r_{mk}\, e^{-u_m(\boldsymbol{\alpha}_k, \mathbf{r}_k)}}
{\sum_{j=1}^{K} r_{mj}\, e^{-u_m(\boldsymbol{\alpha}_j, \mathbf{r}_j)}},
\qquad \sum_{k=1}^{K} q_{mk} = 1.
\end{equation}
Intuitively, a cluster is softly attached to the experts that both
already claim it ($r_{mk}$ large) and are predicted to model
\emph{that cluster} well ($u_m(\boldsymbol{\alpha}_k, \mathbf{r}_k)$ small);
the out-of-region components of Definition~\ref{def:surrogate} are what
let different clusters favour different experts.

\textbf{M-step.} Substituting the E-step posterior into
$\mathbb{E}_q[\log p(\mathbf{y}, \mathbf{Z}\mid\cdot)]$ and grouping the
sample-level terms by cluster yields the $Q$-function
\begin{equation}
\label{eq:qfun_main}
Q(\boldsymbol{\alpha}_{1:K}, \mathbf{R})
\;=\; \sum_{m=1}^{M} |\mathcal{C}_m| \sum_{k=1}^{K} q_{mk}
\bigl[\log r_{mk} \,-\, u_m(\boldsymbol{\alpha}_k, \mathbf{r}_k)\bigr],
\end{equation}
which the M-step maximises over the routing $\mathbf{R}$ and the
per-expert architectures $\boldsymbol{\alpha}_{1:K}$ (maximising $Q$ at
fixed $q$ coincides with maximising $\mathcal{F}_u(q, \cdot)$, as the
entropy $\mathbb{H}[q]$ is independent of $\theta$). The two parameter
blocks are updated in turn.

\emph{M-step (routing).} For fixed architectures, $Q$ is maximised over
$\mathbf{R}$ subject to $\sum_k r_{mk}=1$ by gradient ascent on the
routing logits, using a straight-through Gumbel-Softmax to back-propagate
through the discrete masks $\mathbf{r}_k$ that the surrogate consumes,
which is exactly how the extension~\eqref{eq:mask_extension} is
evaluated. The term
$\sum_{m,k} |\mathcal{C}_m|\, q_{mk}\log r_{mk}$ alone is a cross-entropy
whose constrained optimum is $r_{mk}=q_{mk}$, i.e.\ it pulls the routing
towards the current responsibilities, while the
$-q_{mk}\,u_m(\boldsymbol{\alpha}_k, \mathbf{r}_k)$ term additionally
biases each column $\mathbf{r}_k$ towards masks on which expert $k$ is
predicted to perform well \emph{on the clusters that claim it}: enlarging
$\mathbf{r}_k$ with a foreign cluster raises $u_{m'}$ on the clusters
$m'$ the expert already owns, and the gradient sees that interference.
Because the predicted loss varies with the evaluation cluster, this step
does not collapse onto a single expert.

\emph{M-step (architectures).} For fixed routing the only
architecture-dependent part of~\eqref{eq:qfun_main} is
\begin{equation}
\label{eq:mstep_alpha}
\boldsymbol{\alpha}_k \;\in\; \arg\min_{\boldsymbol{\alpha} \in \mathcal{A}}
\;\sum_{m=1}^{M} |\mathcal{C}_m|\, q_{mk}\, u_m(\boldsymbol{\alpha}, \mathbf{r}_k),
\end{equation}
so each expert is optimised independently by \emph{minimising} its
responsibility-weighted surrogate loss --- a weighted average of the
per-cluster losses that concentrates on the clusters expert $k$ is
responsible for. Since $\boldsymbol{u}$ is non-differentiable in
$\boldsymbol{\alpha}$, this step is solved by sampling: for each expert
$k$ we draw a pool of candidate architectures from $\mathcal{A}$, score
every candidate by~\eqref{eq:mstep_alpha} under the current
responsibilities, and keep the minimiser as
$\boldsymbol{\alpha}_k$. Scoring the pool costs only forward passes of
$\boldsymbol{u}_{\boldsymbol{\theta}}$; the expensive operation is the
active-learning step that follows, in which a small number of the
selected candidates are actually trained. Which of them is trained is
decided by the Lower Confidence Bound of Section~\ref{subsec:surr},
so the same pass both advances the M-step and supplies the surrogate
with the observation it is least certain about. Because the pool is
drawn afresh at every iteration, the M-step is a stochastic
\emph{improvement} rather than an exact maximisation --- which is why
the analysis of Section~\ref{subsec:theory} only assumes a
generalised (improvement) M-step.

\textbf{Exploration of the mask space.} Left to itself, the loop above
under-explores. The masks it trains on all concentrate around the
current routing, so the surrogate never observes near-pure regions and
drifts towards confirming the split it already has --- an effect that is
mild when the domains differ sharply in difficulty and severe when they
do not. We therefore augment the active-learning step with three
mechanisms, driven by the algorithm's own signals and never by any
domain labels. All three rest on \emph{cluster affinities} read off the
observations collected so far: clusters whose per-cluster losses react
similarly to architectures --- after the global effect of an
architecture has been removed --- are grouped by agglomerative
clustering. First, each affinity group is proposed as a training mask in
its own right and evaluated with several architectures; these
\emph{exploration masks} supply the near-pure regions that the plain
acquisition, which keeps sampling masks close to the current routing,
never reaches. Second, the same affinities propose complete candidate
partitions of the clusters into $K$ regions, which lets the routing
escape a local optimum in one \emph{jump} instead of drifting there
cluster by cluster; a candidate is accepted only if it improves the real
--- not the predicted --- objective proxy, so a mis-calibrated surrogate
cannot push the routing to a worse partition. Finally, once the EM loop
has terminated and the routing is fixed, the architecture of each expert
is re-searched by random search restricted to its final region, an
ordinary per-region NAS problem.
Appendix~\ref{app:exploration} details the three mechanisms.

The complete procedure --- alternating the E-step, the two M-steps over
the routing $\mathbf{R}$ and the architectures $\boldsymbol{\alpha}_{1:K}$,
and an \emph{S-step}, the active-learning refit of the surrogate itself
--- is summarised in Algorithm~\ref{alg:sgem}. We name the last one by
analogy with the other two: like the E- and M-steps it is executed once
per iteration, and it is the only one that spends real training
compute.

\begin{algorithm}[h]
\caption{Surrogate-assisted Generalized EM (SGEM) for MoE Architecture Search}
\label{alg:sgem}
\begin{algorithmic}[1]
\Require Surrogate $\boldsymbol{u}_{\boldsymbol{\theta}}$, search space
$\mathcal{A}$, number of experts~$K$, clusters $M$,
EM iterations $T$.
\State Initialise architectures
$\boldsymbol{\alpha}_1, \dots, \boldsymbol{\alpha}_K$ at random and
logits $\boldsymbol{\ell}\!\in\!\mathbb{R}^{M\times K}\!\sim\!\mathcal{N}(\mathbf{0}, I)$.
\For{$t = 1$ to $T$}
    \State \textbf{E-step.} $\;q_{mk} \propto r_{mk}\, e^{-u_m(\boldsymbol{\alpha}_k, \mathbf{r}_k)}$.
    \State \textbf{M-step (r).} Gradient ascent on
    $\sum_m |\mathcal{C}_m|\sum_k q_{mk}[\log r_{mk} - u_m(\boldsymbol{\alpha}_k, \mathbf{r}_k)]$
    using straight-through Gumbel-Softmax for $\mathbf{r}_k$.
    \State \textbf{M-step ($\boldsymbol{\alpha}$).} For each expert
    $k$, sample candidates and keep the one minimising
    $\sum_m |\mathcal{C}_m|\, q_{mk}\, u_m(\cdot, \mathbf{r}_k)$.
    \State \textbf{S-step.} Train selected candidates,
    update $\mathcal{D}_{\text{surr}}$, refit $u_{\boldsymbol{\theta}}$.
\EndFor
\State \Return $\boldsymbol{\alpha}_{1:K}$ and hard assignment
$\hat{r}_{mk} = 1, [k = \arg\max_{k'} r_{mk'}]$.
\end{algorithmic}
\end{algorithm}

\subsection{Convergence under a Drifting Surrogate}
\label{subsec:theory}

We show that, provided the surrogate error decays fast enough, SGEM
converges in the same sense as Generalised EM --- even though the
surrogate $u^{(t)}$ used inside the E- and M-steps is itself updated
across iterations.

Let $\boldsymbol{u}_{\mathrm{true}}$ denote the (unknown) ground-truth
expert loss: at a binary mask $\boldsymbol{b}$ its $m$-th component is
the per-sample NLL on cluster $m$ actually attained by training
architecture $\boldsymbol{\alpha}$ on $\mathcal{D}_{\boldsymbol{b}}$,
and it is carried to fractional $\mathbf{r}$
by~\eqref{eq:mask_extension}, exactly as the surrogate is. Define the
per-iteration surrogate error directly in the loss scale, as the
sup-norm over both arguments and all output components,
\begin{equation}
\label{eq:eps_t}
\varepsilon_t \;=\; \sup_{\boldsymbol{\alpha} \in \mathcal{A}, \;
\mathbf{r} \in [0, 1]^M, \; m \in \{1, \ldots, M\}}
\bigl|u^{(t)}_m(\boldsymbol{\alpha}, \mathbf{r})
\;-\; u_{\mathrm{true}, m}(\boldsymbol{\alpha}, \mathbf{r})\bigr|.
\end{equation}
Whenever both maps are read through the
extension~\eqref{eq:mask_extension} --- as they are in the M-step --- the
supremum in~\eqref{eq:eps_t} is attained at a binary mask, so
$\varepsilon_t$ is the surrogate error on the $2^M$ masks, precisely the
set on which the surrogate can be given evidence.
The parameter $\theta = (\boldsymbol{\alpha}_{1:K}, \mathbf{R})$ lies
on the compact set
$\Theta = \mathcal{A}^K \times (\Delta_K)^M$; we write
$N = \sum_m |\mathcal{C}_m|$ for the total sample count. The (size-weighted)
true log-likelihood is
\begin{equation}
\label{eq:Ltrue}
L_{\mathrm{true}}(\theta) \;=\; \sum_{m=1}^M |\mathcal{C}_m|\,\log
\sum_{k=1}^K r_{mk}\, e^{-u_{\mathrm{true}, m}(\boldsymbol{\alpha}_k, \mathbf{r}_k)},
\end{equation}
and the \emph{true} Q-function is
\begin{equation}
\label{eq:Qtrue}
Q_{\mathrm{true}}(\theta; \theta') \;=\; \sum_{m=1}^M |\mathcal{C}_m|
\sum_{k=1}^K w_{mk}(\theta')
\bigl[\log r_{mk} \,-\, u_{\mathrm{true}, m}(\boldsymbol{\alpha}_k, \mathbf{r}_k)\bigr],
\end{equation}
where $w_{mk}(\theta')$ are the responsibilities computed under
$\boldsymbol{u}_{\mathrm{true}}$ (Proposition~\ref{prop:elbo} in
Appendix~\ref{sec:proof}).

Note that the analysis below never uses the dependence of the loss on the
evaluation cluster $m$ --- it only requires that each component be bounded
and that the error~\eqref{eq:eps_t} be summable. Theorem~\ref{thm:sgem}
therefore holds for any surrogate satisfying these two conditions.

\begin{assumption}
\label{ass:bounded}
There exists $U_{\max} < \infty$ such that
$0 \leq u_{\mathrm{true}, m}(\boldsymbol{\alpha}, \mathbf{r}),\,
u^{(t)}_m(\boldsymbol{\alpha}, \mathbf{r}) \leq U_{\max}$ for all
$\boldsymbol{\alpha}\in\mathcal{A}$, $\mathbf{r}\in[0,1]^M$,
$m \in \{1, \ldots, M\}$, $t\geq 0$. Since~\eqref{eq:mask_extension} is an
average of vertex values, a bound on the binary masks already implies the
bound on the whole cube. In practice the surrogate targets are
clipped at an upper quantile, which enforces it: an expert evaluated
far outside its training region can otherwise incur an arbitrarily large NLL.
\end{assumption}

\begin{assumption}
\label{ass:summable}
The surrogate errors are summable: $\sum_{t=0}^{\infty} \varepsilon_t < \infty$.
\end{assumption}

\begin{assumption}
\label{ass:gem}
At every iteration $t$ the M-step of Algorithm~\ref{alg:sgem} maximises
the free energy over a \emph{candidate set}
$\mathcal{M}^{(t)} \subseteq \Theta$ that contains the incumbent,
$\theta^{(t)} \in \mathcal{M}^{(t)}$:
\begin{equation*}
\theta^{(t+1)} \;\in\; \arg\max_{\theta \in \mathcal{M}^{(t)}}
\mathcal{F}_{u^{(t)}}\bigl(q^{(t)}, \theta\bigr).
\end{equation*}
In Algorithm~\ref{alg:sgem} the set $\mathcal{M}^{(t)}$ consists of the
architectures drawn into the pool at iteration $t$ paired with the
routings visited by the gradient update, and keeping the incumbent in
the pool is what makes the inclusion $\theta^{(t)} \in \mathcal{M}^{(t)}$
hold. The step is then \emph{generalised} in the usual sense,
\begin{equation*}
\mathcal{F}_{u^{(t)}}\bigl(q^{(t)}, \theta^{(t+1)}\bigr)
\;\geq\; \mathcal{F}_{u^{(t)}}\bigl(q^{(t)}, \theta^{(t)}\bigr),
\end{equation*}
without being an exact maximisation over $\Theta$. Here the variational
free energy is
\begin{equation*}
\mathcal{F}_u(q, \theta) \;=\; \sum_{m=1}^M |\mathcal{C}_m|
\sum_{k=1}^K q_{mk}\,\log\!\frac{r_{mk}\, e^{-u_m(\boldsymbol{\alpha}_k, \mathbf{r}_k)}}{q_{mk}}.
\end{equation*}
\end{assumption}

\begin{assumption}
\label{ass:cont}
For every $\boldsymbol{\alpha}\in\mathcal{A}$ and every
$m \in \{1, \ldots, M\}$ the map
$\mathbf{r}\mapsto u_{\mathrm{true}, m}(\boldsymbol{\alpha}, \mathbf{r})$
is continuous on $[0,1]^M$. For the canonical
extension~\eqref{eq:mask_extension} this holds by construction, the
extension being a polynomial in $\mathbf{r}$; we state it separately
because continuity is all the proof uses.
\end{assumption}

\begin{theorem}[Convergence of SGEM with adaptive surrogate]
\label{thm:sgem}
Under Assumptions~\ref{ass:bounded}--\ref{ass:cont}, the iterates
$\{\theta^{(t)}\}$ produced by Algorithm~\ref{alg:sgem} satisfy
\begin{enumerate}
\item[\textup{(a)}] the sequence $\{L_{\mathrm{true}}(\theta^{(t)})\}$ converges;
\item[\textup{(b)}] every limit point $\theta^{\star}$ is a fixed point
of the M-step with respect to the \emph{true} objective: writing
$\theta^{(t_j)} \to \theta^{\star}$ for a convergent subsequence and
\begin{equation*}
\mathcal{M}^{\infty}(\theta^{\star}) \;=\;
\bigl\{\theta' \in \Theta \;:\; \theta' \in \mathcal{M}^{(t_j)}
\text{ for infinitely many } j \bigr\}
\end{equation*}
for the candidates the M-step keeps offering near $\theta^{\star}$, one has
$Q_{\mathrm{true}}(\theta';\theta^{\star}) \leq
Q_{\mathrm{true}}(\theta^{\star};\theta^{\star})$
for every $\theta' \in \mathcal{M}^{\infty}(\theta^{\star})$.
\end{enumerate}
\end{theorem}

Part~(b) is stated relative to what the M-step can reach, because
Algorithm~\ref{alg:sgem} scores a freshly drawn pool of architectures
instead of all of $\mathcal{A}^K$ and moves $\mathbf{R}$ by gradient
ascent: no candidate the algorithm actually produces improves the true
$Q$-function at $\theta^{\star}$. Were the M-step exhaustive
($\mathcal{M}^{(t)} = \Theta$), this would be the classical stationarity
condition of EM. Short of that, drawing the pool i.i.d.\ from a
distribution with full support on the finite set $\mathcal{A}^K$ already
makes every architecture vector enter the pool infinitely often almost
surely, by the Borel--Cantelli lemma, so that
$\mathcal{M}^{\infty}(\theta^{\star})$ contains every architecture; for
the routing block, updated by gradient ascent, part~(b) is the
corresponding first-order condition.

The proof, together with the auxiliary lemmas (uniform deviation of
the ELBO, stability of the softmax responsibilities, and a
quasi-monotonicity bound
$L_{\mathrm{true}}(\theta^{(t+1)}) \geq L_{\mathrm{true}}(\theta^{(t)})
- 2N\varepsilon_t$ established via a Robbins--Siegmund argument),
is given in Appendix~\ref{sec:proof}.

\begin{remark}
The summability assumption~\ref{ass:summable} is the surrogate
analogue of the diminishing step-size condition in stochastic
EM~\cite{dempster1977em}; replacing it by the weaker condition
$\varepsilon_t\to 0$ preserves part~(b) of Theorem~\ref{thm:sgem} but
not part~(a). Practically, summability is achieved whenever the
S-step of Algorithm~\ref{alg:sgem} provides
sufficiently rich coverage of $\mathcal{A}\times[0,1]^M$.
\end{remark}

\textbf{Novelty and relation to classical convergence results.} The
classical theory of EM and its generalised
variant~\cite{dempster1977em, wu1983convergence} rests on the monotone
ascent of a single, \emph{fixed} likelihood --- a premise our setting
violates, since the objective optimised at iteration $t$ is read off the
re-fitted surrogate $u^{(t)}$ and is thus a moving target, while
$L_{\mathrm{true}}$ is never evaluated by the algorithm at all.
Theorem~\ref{thm:sgem} retains convergence nonetheless, with the
quasi-monotonicity bound above taking over the role of exact
monotonicity; this is the main theoretical contribution. It also differs
from stochastic and online EM analyses~\cite{cappe2009online}, where the
perturbation of the exact update is \emph{zero-mean stochastic noise}
tamed by a diminishing step-size schedule: here it is a
\emph{deterministic, bounded approximation error} of a learned
surrogate, and the role of the vanishing step size is played by the
summability condition~\ref{ass:summable}, whose smallness is secured by
active-learning coverage of the search space rather than by averaging
out noise. To our knowledge, this is the first convergence guarantee for
an EM-type architecture-search procedure whose per-iteration objective
is supplied by an \emph{adaptively trained} surrogate model.
 
\section{Experiments}
\label{sec:experiments}

We evaluate the proposed framework on three complementary problems:
(i) a controlled 2D toy task that allows us to verify whether the
recovered routing matches the ground-truth cluster structure,
(ii) a heterogeneous image-classification benchmark obtained by mixing
CIFAR-100 with SVHN, whose two sources give a sharp ground-truth domain
structure and hence a well-defined oracle routing, where we compare
against several MoE and NAS baselines, and (iii) a multi-domain time-series forecasting
benchmark that probes the framework in a harder, low-contrast regime
against a ladder of routing and architecture baselines.

\subsection{Toy Experiment}
\label{subsec:toy}

We sample $2000$ points in $\mathbb{R}^2$ from a four-class problem
consisting of one ``linear'' cloud (two classes separated by a
diagonal) and one ``ring'' cloud (two classes corresponding to an
inner and an outer ring), shown in Figure~\ref{fig:toy_dataset}. The
data is split $80/20$ into training/validation and clustered into
$M = 20$ $k$-means clusters (Figure~\ref{fig:toy_clusters}).
We use $K = 2$ experts so that the ideal solution is unambiguous: one
expert should specialise to the linear cloud and the other to the
ring cloud. The search space is a small cell-based space over the
$2$D inputs, in the spirit of \cite{liu2018darts}, whose operation set
is deliberately chosen so that the two clouds prefer \emph{different}
architectures: a \emph{linear} map
$\boldsymbol{x}\mapsto W\boldsymbol{x}$ already separates the linear
cloud but cannot model concentric rings, while a \emph{radial basis
function}
$\boldsymbol{x}\mapsto \exp\!\bigl(-\gamma\,\lVert\boldsymbol{x}-\boldsymbol{c}\rVert^2\bigr)$
separates the rings by radius and is needlessly poor on the linear
cloud (identity, element-wise squaring, additive shift, $\mathrm{ReLU}$
and dropout complete the set). The optimal mixture is therefore forced
to assign \emph{architecturally distinct} experts to the two regions,
so the toy task probes architectural specialisation, not merely the
routing.

SGEM (Algorithm~\ref{alg:sgem}) recovers the \emph{ideal} cluster split
($7$ ring / $13$ linear clusters) between the two experts in $20$ EM iterations
(Figure~\ref{fig:toy_sgem}), matching
the hand-defined oracle assignment that splits clusters according to
the centroid coordinate $x_0 = -2$. This sanity check confirms that
the surrogate-guided objective in~\eqref{eq:moe_objective} is well
posed and that the proposed optimiser solves it when the surrogate
is sufficiently accurate, in agreement with the convergence
guarantee of Theorem~\ref{thm:sgem}.

\begin{figure}[h]
  \centering
  \begin{subfigure}[b]{0.32\textwidth}
      \centering
      \includegraphics[width=\textwidth]{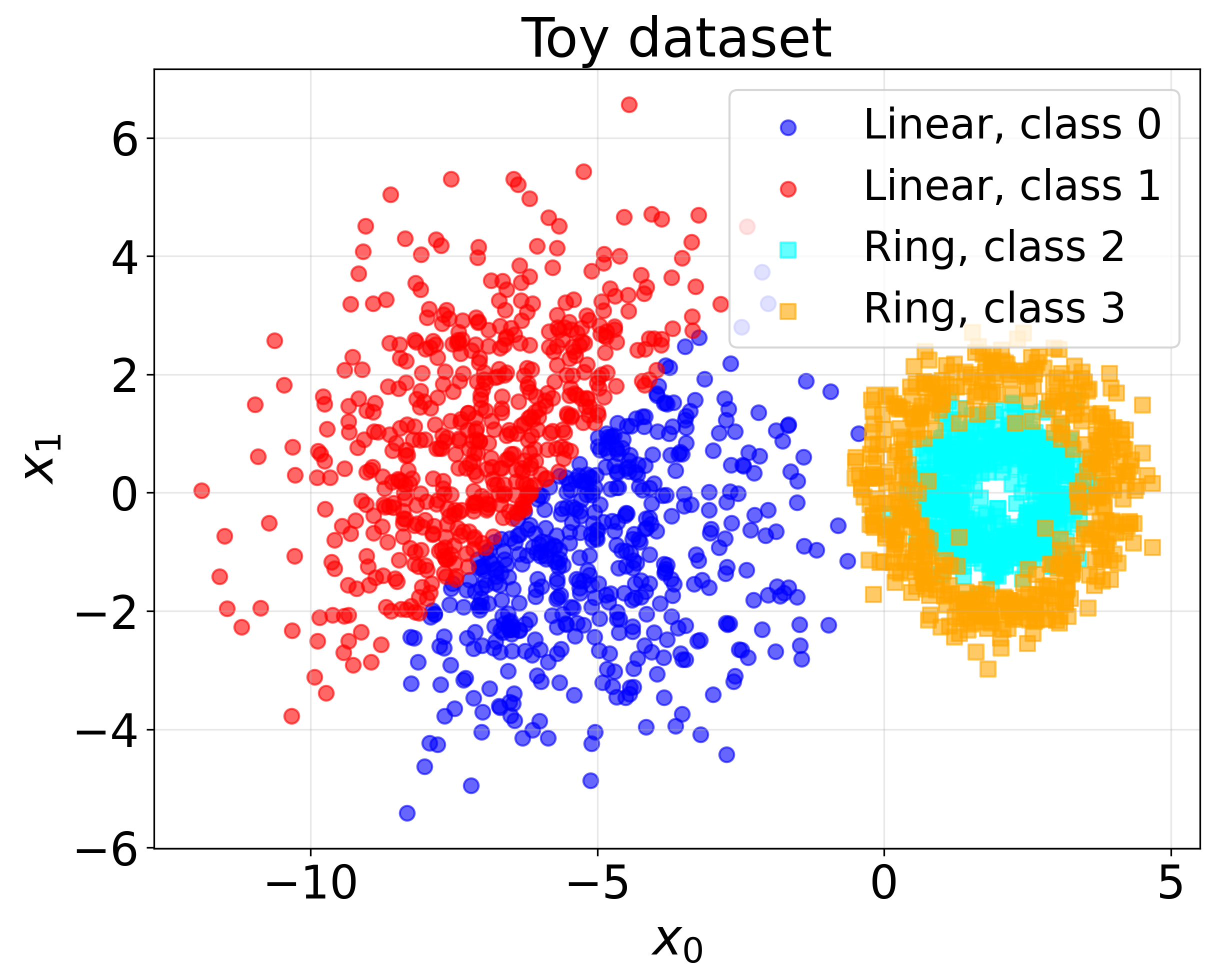}
      \captionsetup{justification=centering}
      \caption{}
      \label{fig:toy_dataset}
  \end{subfigure}
  \hfill
  \begin{subfigure}[b]{0.32\textwidth}
      \centering
      \includegraphics[width=\textwidth]{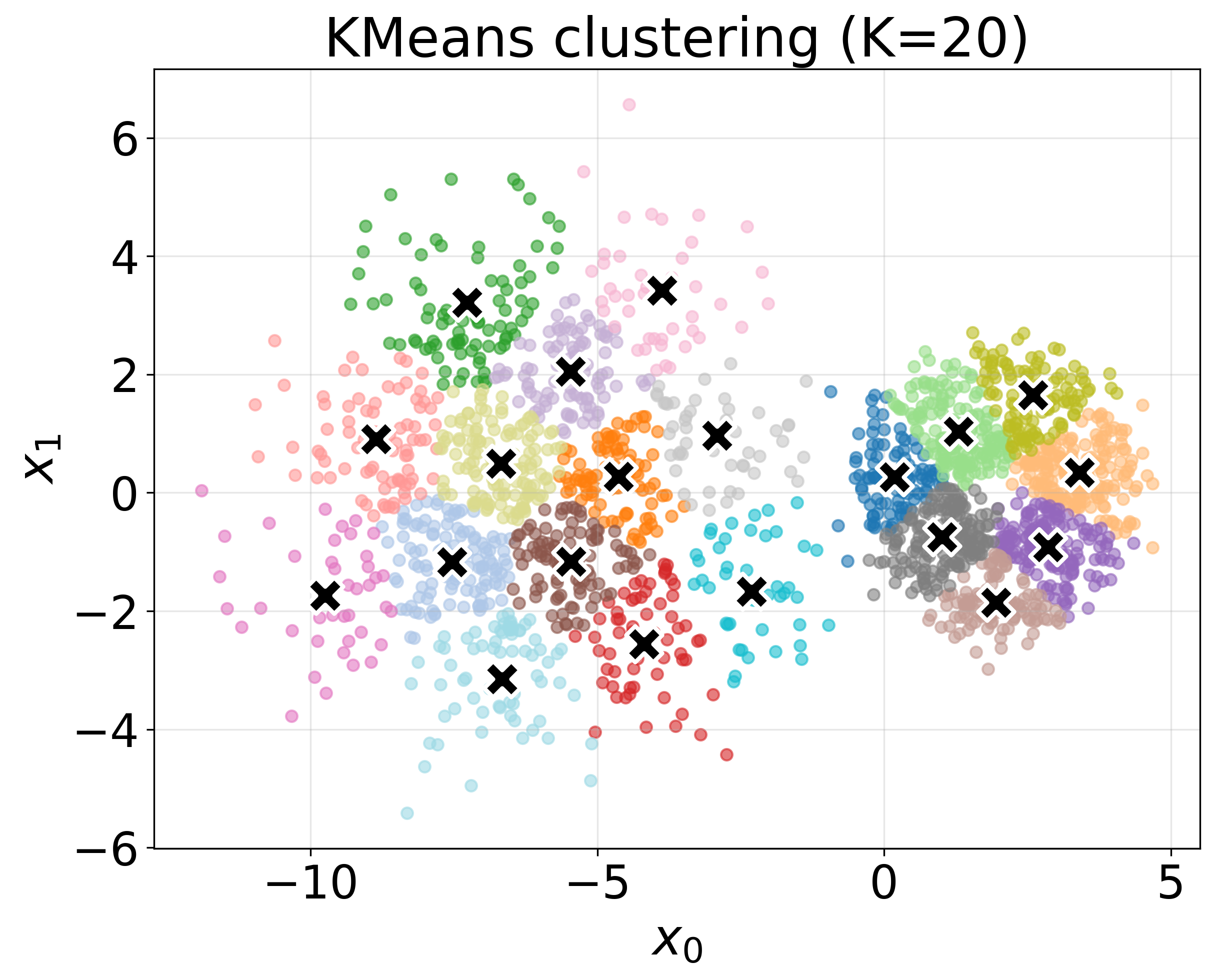}
      \captionsetup{justification=centering}
      \caption{}
      \label{fig:toy_clusters}
  \end{subfigure}
  \hfill
  \begin{subfigure}[b]{0.32\textwidth}
      \centering
      \includegraphics[width=\textwidth]{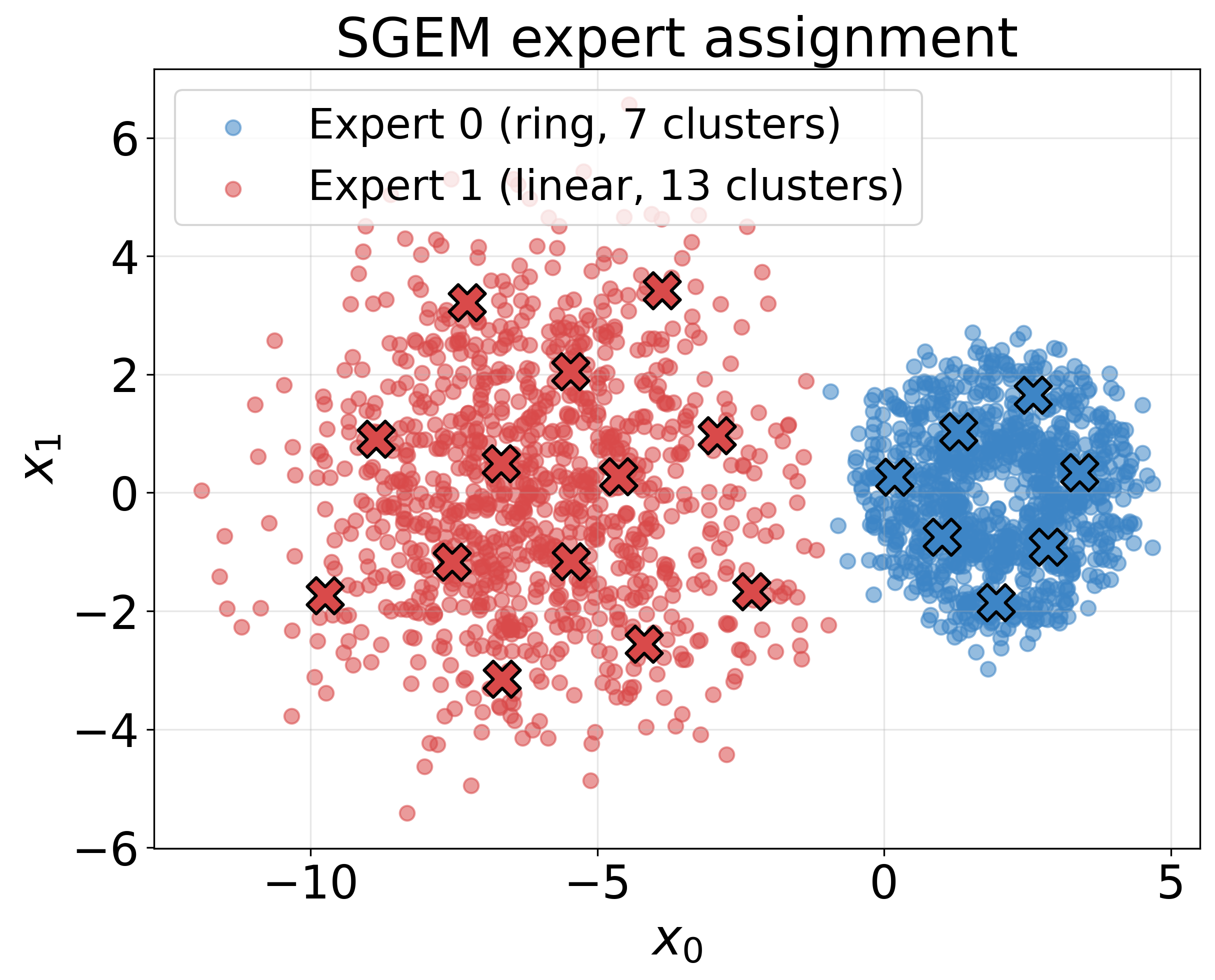}
      \captionsetup{justification=centering}
      \caption{}
      \label{fig:toy_sgem}
  \end{subfigure}
  \caption{Toy experiment. (a) The four-class dataset: a linear cloud
  (classes 0--1) and a ring cloud (classes 2--3). (b) The $M = 20$
  $k$-means clusters (crosses mark the centroids). (c) The expert
  assignment recovered by SGEM with $K = 2$: it splits the clusters
  exactly along the linear/ring boundary, matching the oracle
  partition up to a permutation of expert labels.}
  \label{fig:toy}
\end{figure}

\subsection{CIFAR-100 and SVHN mixture}
\label{subsec:cifar}

To test whether the cluster-aware objective recovers a meaningful
data partition on realistic data, we build a heterogeneous benchmark
by mixing CIFAR-100 and SVHN in equal proportions ($42\,000$ images
total). SVHN labels are shifted by $100$, so the combined problem has
$110$ classes. The two sources form a maximally separated domain
structure, which gives a well-defined oracle: with $K = 2$ experts the
ideal routing assigns each cluster to the expert specialised on its
source. We partition the data into $M = 30$ \emph{semantic} clusters
obtained from a pretrained feature extractor, using a three-way
train/validation/test split ($0.7/0.15/0.15$). The clustering is fit on the
training embeddings only; validation and test points are assigned to
the nearest centroid. The resulting oracle partition is
$19$ CIFAR-majority and $11$ SVHN-majority clusters, with a mean
source purity of $0.978$.

All MoE configurations use $K = 2$ experts with
\texttt{init\_channels = 16}. The base network is a small DARTS-style
architecture \texttt{CIFAR100Net} (stem, two normal cells, one fixed
reduction cell, global average pooling, fully connected head). The MoE
wrapper instantiates $K$ such networks combined through a fixed hard
\emph{cluster gate}: each cluster is routed to a single expert, so the
gate directly reflects the discovered data partition. Architecture
search uses the train/validation split; the final model is retrained
for $100$ epochs on $\text{train} \cup \text{validation}$ and evaluated
\emph{once} on the held-out test set. We report test accuracy averaged
over five random seeds. Training is performed inside a Docker
container with PyTorch~2.5.1 on a single RTX~3080~Ti GPU.

We compare against the following methods, all evaluated with the same
cluster gate:
\begin{itemize}
    \item \textbf{Random-MoE}: $200$ random architecture pairs are
    searched at $30$ epochs each on a random cluster split, then the
    best is retrained.
    \item \textbf{DARTS$\times K$}: $K$ experts that share the same
    DARTS-discovered architecture on a random cluster split.
    \item \textbf{SGEM} (ours): Algorithm~\ref{alg:sgem} with the
    exploration mechanisms and final refit
    of Section~\ref{subsec:sgem}, and the two
    data-collection measures described next. No
    source labels are used at any stage.
    \item \textbf{DARTS oracle MoE}: an \emph{upper bound} in which the
    cluster$\to$expert routing is fixed to the true source split and each
    expert uses a DARTS architecture searched on its own source. It uses
    the source labels and is therefore not a competitor but a ceiling on
    what perfect routing can achieve.
\end{itemize}

\textbf{Surrogate data collection.} Two obstacles stand between the
surrogate dataset and a usable routing signal, and each is met by a
domain-agnostic choice. \emph{Evaluation leakage}: scoring a candidate
on the whole validation set mixes in clusters it was never trained for
and collapses the predicted-loss gap between a specialised and a mixed
expert, which is the signal the routing relies on; we therefore
evaluate each candidate only on its own clusters. \emph{Mask
concentration}: including every cluster with probability $\tfrac12$
concentrates masks around $M/2$, so near-pure single-source subsets are
essentially never drawn ($2^{-19}$ for pure SVHN) and the surrogate
cannot learn that a pure expert attains a low loss; we therefore sample
the mask size \emph{log-uniformly} (minimum three clusters). Neither
measure consults the source labels, and both are needed: without them
accuracy falls to $\approx\!0.59$ at near-random source recovery
($\approx\!60\%$).

\begin{figure}[h]
    \centering
    \includegraphics[height=0.44\textheight]{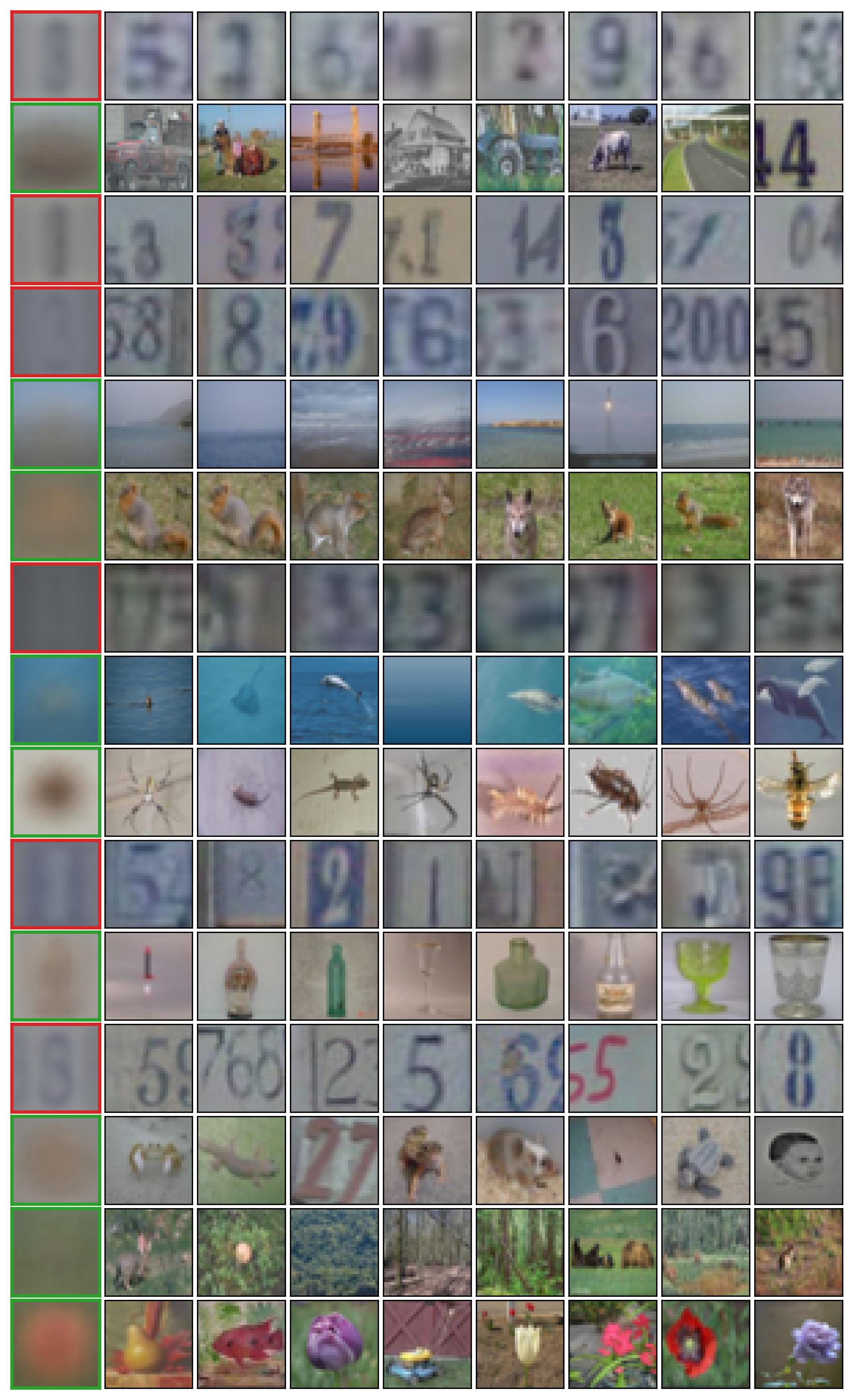}
    \caption{Per-cluster representatives of the CIFAR-100 / SVHN
    mixture under the SGEM routing ($K=2$, seed $322$; first $15$ of
    the $M = 30$ semantic clusters shown). Each row is one cluster: the
    leftmost column is the cluster medoid and the rest are the eight
    images nearest to the cluster mean. The medoid border is coloured by
    the expert SGEM assigns to the cluster (E0, the SVHN specialist, red;
    E1, the CIFAR specialist, green). SGEM sends every SVHN (digit)
    cluster to E0 and keeps E1 a pure CIFAR expert; the source-recovery
    rate is $93\%$ for this seed.}
    \label{fig:svhn_reps}
\end{figure}

Table~\ref{tab:cifar_main} reports test accuracy with cluster gating.
SGEM reaches $0.6402 \pm 0.002$ over five seeds,
outperforming DARTS$\times K$ ($0.5913 \pm 0.009$) and the random-MoE
baseline ($0.5551 \pm 0.031$) by $+4.9$ and $+8.5$ percentage points, and
even exceeding the DARTS oracle MoE ($0.6317 \pm 0.004$). The oracle row
is \emph{given} the true source split but keeps DARTS-searched expert
architectures, so it bounds what routing alone can achieve at a fixed
architecture quality; the per-region final refit of SGEM finds stronger
expert architectures and thereby passes it. The cluster gate is what
isolates routing quality: under a learnable gate the methods become
nearly indistinguishable (all $\approx\!0.63$), because a trainable gate
can compensate for a suboptimal partition and so masks the benefit of a
good one. A finer ablation that decomposes this accuracy into a routing
and an architecture contribution is reported in
Appendix~\ref{app:decomp}: the gain stems almost entirely from the
cluster$\to$expert routing, while architectural heterogeneity between
experts adds little once the routing is correct, and learning the
routing inside SGEM (rather than freezing it) is what recovers the
source structure.

\begin{table}[h]
\centering
\caption{Test accuracy on the CIFAR-100 / SVHN mixture with cluster
gating ($M=30$, $K=2$, semantic clustering, mean $\pm$ std over five
seeds). The DARTS
oracle MoE uses the ground-truth source split with DARTS-searched
architectures: it is a routing upper bound at fixed architecture
quality, not a competitor.}
\label{tab:cifar_main}
\begin{tabular}{lc}
\toprule
Method & test acc.\ (cluster gate) \\
\midrule
Random-MoE & $0.5551 \pm 0.031$ \\
DARTS$\times K$ (same arch) & $0.5913 \pm 0.009$ \\
\textbf{SGEM (ours)} & $\mathbf{0.6402 \pm 0.002}$ \\
\midrule
DARTS oracle MoE \emph{(uses source labels)} & $0.6317 \pm 0.004$ \\
\bottomrule
\end{tabular}
\end{table}

Because the oracle partition by data source is sharp (mean purity
$0.978$), we can directly check whether the learned routing aligns with
it. Across five seeds the SGEM hard assignment agrees with
the oracle source split on $95\% \pm 2\%$ of clusters (best of the two
label permutations) --- far above the $\approx\!54\%$ of a balanced
random split and a large jump over the $\approx\!60\%$ reached without
the two data-collection measures above. In a typical run (seed $322$, Figure~\ref{fig:svhn_reps}) one
expert becomes a \emph{pure} CIFAR specialist ($17$ clusters, no SVHN)
while the other absorbs all $11$ SVHN clusters together with only $2$
CIFAR ones. The cluster-aware objective therefore \emph{does} discover
the underlying data domains on this benchmark; the small residual gap to the oracle stems from a few
genuinely mixed clusters and from objective noise near the optimum: the
iterate selected by best objective occasionally misses a perfect
assignment --- all $30$ clusters matching the oracle --- that was
attained one iteration earlier.

Empirically the routing converges fast: on every seed the
source-recovery rate jumps from its initial $53$--$67\%$ to
$\approx\!97\%$ within the first two EM iterations and then remains on a
$0.93$--$0.97$ plateau, illustrating the convergence guaranteed by
Theorem~\ref{thm:sgem}; Appendix~\ref{app:convergence} traces a single
run iteration by iteration.

\subsection{Time-Series Forecasting over a Multi-Domain Mixture}
\label{subsec:ts}

The CIFAR-100/SVHN benchmark has two domains of very different
difficulty, which makes the domain structure easy for the likelihood to
reward. To probe the framework in a harder, low-contrast regime, we
turn to probabilistic
time-series forecasting over a mixture of four Monash
\cite{godahewa2021monash} domains: hourly electricity demand, air
quality (KDD Cup 2018), weather, and monthly tourism. We sample
$40\,000$ sliding windows with context length $L = 96$ and forecast
horizon $H = 24$, and cluster them into $M = 30$ $k$-means clusters
computed from simple statistical descriptors of each window (no
pretrained embeddings and no domain labels). The clusters align with the
domains only approximately (mean source purity $0.864$; the oracle
partition has $5/5/8/12$ clusters per domain), and, unlike SVHN, all
four domains are forecasting tasks of comparable difficulty. We use
$K = 4$ experts. Each expert is a small cell-based forecasting network
drawn from a search space over the six time-series operations of
Table~\ref{tab:ts_ops} (moving average, dilated convolution, GRU, patch
attention, a linear trend and a skip connection) that emits
the parameters of a per-step Student-$t$ predictive distribution. The
data protocol mirrors Section~\ref{subsec:cifar}: a three-way
train/validation/test split, search on train/validation only, and a
final MoE retrained for $50$ epochs on
$\text{train} \cup \text{validation}$ and evaluated \emph{once} on the
held-out test set. We report the test negative log-likelihood (NLL) and
the sample-based continuous ranked probability score (CRPS), both
averaged over $5$ seeds (lower is better).

\textbf{Exploration.} Because the four domains are of similar
difficulty, this benchmark is where the exploration mechanisms of
Section~\ref{subsec:sgem} matter most: without them the masks
concentrate around the current routing and the surrogate drifts towards
confirming it. All runs below use the exploration masks, the partition
jumps and the final per-region refit, with no domain labels at any
stage.

\textbf{Baselines.} We compare against a ladder of baselines that
isolate the individual design choices; all MoE methods use $K = 4$ and
the same final-training protocol, and all label-free methods see no
domain information. \emph{Global single} trains one network ($K = 1$)
found by random search. \emph{Random-MoE} uses random architectures on
a random cluster partition. \emph{BestArch$\times K$} uses $K$ copies
of the best single architecture (the analogue of DARTS$\times K$ in
Section~\ref{subsec:cifar}). \emph{Geometry routing} assigns clusters
to experts by $k$-means over cluster centroids, with $K$ copies of the
best single architecture. \emph{NAS-MoE} equips a random fixed
partition with $K$ independently searched architectures.
\emph{MoLE}~\cite{ni2024mole} is a mixture of linear experts with a
learnable gate. Finally, an \emph{oracle} method uses the ground-truth
domain labels and bounds what perfect routing can achieve:
\emph{per-domain specialists} route each domain to its own expert and
search a separate architecture for every domain. Simple
statistical baselines (naive, seasonal-naive, DLinear) are far behind
all MoE variants (test NLL $2.01$, $1.91$ and $1.80$ respectively) and
are omitted from the table.

\begin{table}[h]
\centering
\caption{Test NLL and CRPS on the four-domain Monash mixture
($M = 30$, $K = 4$, mean $\pm$ std over $5$ seeds; lower is better).
Bold marks the best entry of the first block, whose methods all use the
cluster gate and no domain labels and are therefore directly comparable.
MoLE is separated out because it routes with its own learnable gate; the
oracle row uses the ground-truth domain labels and is an upper bound,
not a competitor. NLL is the primary metric here, the framework being
likelihood-based throughout.}
\label{tab:ts_main}
\setlength{\tabcolsep}{4.5pt}
\begin{tabular}{lcc}
\toprule
Method & NLL & CRPS \\
\midrule
\multicolumn{3}{l}{\textit{cluster gate, no domain labels}} \\
\quad Global single ($K{=}1$) & $1.070 \pm .038$ & $0.627 \pm .022$ \\
\quad Random-MoE & $1.026 \pm .100$ & $0.630 \pm .045$ \\
\quad Geometry routing & $1.058 \pm .018$ & $0.619 \pm .008$ \\
\quad BestArch$\times K$ & $0.925 \pm .039$ & $0.588 \pm .023$ \\
\quad NAS-MoE & $0.908 \pm .038$ & $0.580 \pm .012$ \\
\quad \textbf{SGEM (ours)} & $\mathbf{0.883 \pm .026}$ & $\mathbf{0.574 \pm .013}$ \\
\midrule
\multicolumn{3}{l}{\textit{own learnable routing}} \\
\quad MoLE~\cite{ni2024mole} & $1.458 \pm .016$ & $0.558 \pm .007$ \\
\midrule
\multicolumn{3}{l}{\textit{oracle, uses domain labels}} \\
\quad Per-domain specialists & $0.789 \pm .028$ & $0.549 \pm .015$ \\
\bottomrule
\end{tabular}
\end{table}

Table~\ref{tab:ts_main} summarises the results. Under the cluster
gate --- the metric that isolates routing quality --- SGEM is the best
method that uses no domain labels: it improves on
BestArch$\times K$ by $0.042$ NLL, on NAS-MoE by $0.025$, and on
geometry routing by a wide $0.175$ margin, showing that likelihood-driven
routing extracts structure that centroid geometry does not.

\textbf{Where the two scores disagree.} MoLE is the one entry that is
not dominated: it attains the lowest CRPS outside the oracle ($0.558$
against $0.574$ for SGEM) while being the worst on NLL by at least
$0.38$ nats. The two scores are not interchangeable --- CRPS is computed
from samples and rewards a well-placed predictive median with a roughly
correct spread, whereas NLL is evaluated on the predictive density and
additionally punishes a mis-calibrated shape --- and we treat NLL as the
primary criterion, since it measures precisely what the
likelihood-based objective of Section~\ref{subsec:objective} optimises.
The comparison is indirect in any case, MoLE bringing its own learnable
gate rather than the cluster gate of the rest of the table
(Section~\ref{subsec:cifar}).

\textbf{What each domain asks of the architecture.} The search already
leaves behind the material for a post-hoc answer, at no extra training
cost: every surrogate observation is a triple (architecture, training
mask, per-cluster loss). Restricting to cluster/observation pairs in
which the cluster was part of the training mask, and standardising the
losses \emph{within} each domain --- which removes the overall strength
of an architecture and leaves only the domain's preference --- the
effect of an operation is the difference in mean standardised loss
between architectures that contain it and those that do not. Negative
means the operation helps. Table~\ref{tab:ts_ops} reports these effects
over the $843$ observations collected across two independent seeds
($6226$ pairs); bold marks effects whose $95\%$ bootstrap interval
excludes zero.

\begin{table}[h]
\centering
\caption{Effect of each operation on the standardised region loss, per
domain (negative = the operation helps). Bold: $95\%$ bootstrap interval
over observations excludes zero. The last row is how much the domain
depends on the architecture at all: the spread of the real region NLL
over $50$ random architectures under oracle routing.}
\label{tab:ts_ops}
\setlength{\tabcolsep}{5pt}
\begin{tabular}{lcccc}
\toprule
Operation & electricity & KDD Cup & weather & tourism \\
\midrule
dilated conv    & $\mathbf{-0.59}$ & $-0.00$          & $\mathbf{+0.15}$ & $\mathbf{-0.21}$ \\
moving average  & $\mathbf{-0.46}$ & $\mathbf{-0.27}$ & $+0.03$          & $\mathbf{-0.35}$ \\
GRU             & $\mathbf{-0.43}$ & $-0.06$          & $\mathbf{-0.10}$ & $-0.06$ \\
patch attention & $\mathbf{-0.23}$ & $\mathbf{-0.40}$ & $\mathbf{-0.12}$ & $-0.03$ \\
linear trend    & $+0.07$          & $\mathbf{+0.20}$ & $\mathbf{-0.12}$ & $\mathbf{+0.22}$ \\
skip connection & $\mathbf{+0.49}$ & $+0.05$          & $+0.04$          & $\mathbf{+0.22}$ \\
\midrule
architecture spread (std NLL) & $0.431$ & $0.118$ & $0.095$ & $0.128$ \\
\bottomrule
\end{tabular}
\end{table}

The preferences are readable and, for some pairs of domains, opposed.
Hourly electricity --- by far the most architecture-sensitive domain,
with a threefold spread between the best and the worst architecture ---
wants a wide receptive field and memory (dilated convolution, GRU,
moving average) and is hurt badly by a bare skip connection, whereas
weather is the one domain that dilated convolution \emph{hurts} and a
linear trend helps. Ranking the six operations by their effect, the two
domains disagree ($\rho = -0.37$), as do weather and tourism
($\rho = -0.26$), while electricity and tourism nearly coincide
($\rho = +0.94$). The disagreement is nevertheless partial: the useful
operations overlap enough --- moving average helps three domains out of
four --- that a single well-chosen architecture remains adequate
everywhere.

Notably, no such disagreement exists on CIFAR-100/SVHN. Repeating the
analysis there, the two sources rank the four operations identically
($\rho = +1$): both prefer the $5\times5$ separable convolution and both
suffer from a bare skip connection, and only the magnitude differs
(SVHN reacts about seven times more strongly). On that benchmark the
architectural preference is a property of the task, not of the domain,
which is consistent with the routing --- and not the architecture ---
carrying the gain there as well.

Unlike on CIFAR-100/SVHN, however, SGEM does not close the gap to the
oracle ceiling on the test set ($0.883$ vs.\ $0.789$ NLL), and its hard
routing recovers the domain partition only marginally better than
chance ($45\%$ mean Hungarian agreement against a $40\%$ random-split
base). This is not an optimiser failure but a property of the
benchmark: an equal-budget ablation shows that the \emph{real}
objective values of the partitions found by SGEM match or exceed the
value of the oracle domain partition on all five seeds, i.e.\ the
likelihood landscape of this low-contrast mixture has near-optimal
partitions that do not coincide with the domain boundaries (the
domain-diagonal NLL gap is only $\approx\!0.09$ nats at $M = 30$ and
vanishes under a finer clustering). The benchmark thus delineates the
scope of domain discovery: when domains differ sharply in difficulty
(CIFAR vs.\ SVHN), the likelihood optimum coincides with the domain
partition and SGEM recovers it; when all domains are comparably hard,
many partitions are equally good and domain recovery is no longer the
right success criterion, while the end-metric advantage of SGEM over
its label-free competitors persists.
 
\section{Conclusion}
\label{sec:conclusion}

We presented a structure-aware architecture search framework for
Mixture-of-Experts models, in which the assignment of data clusters to
experts is treated as a search variable alongside the per-expert
architectures. Writing the cluster-aware objective as the
incomplete-data likelihood of a latent-variable mixture turns the search
into an Expectation--Maximisation problem, solved by SGEM
(Algorithm~\ref{alg:sgem}) with the otherwise intractable expert-quality
term supplied by a vector-valued surrogate refined inside an
active-learning loop. Theorem~\ref{thm:sgem} shows that the procedure
converges even though that objective is re-fitted at every iteration and
the true likelihood is never evaluated by the algorithm.

SGEM recovers the ideal split on the controlled toy problem. On the
deliberately heterogeneous CIFAR-100 / SVHN mixture it recovers the
ground-truth source partition on $95\%$ of clusters \emph{without ever
seeing the source labels} and attains $0.640$ cluster-gate test
accuracy, ahead of every label-free baseline and of the oracle MoE
($0.632$) that is handed the true split but keeps DARTS-searched expert
architectures; region-restricted evaluation and a domain-agnostic
log-uniform mask sampler are what expose the surrogate to the near-pure
subsets this requires, and without them the routing collapses. On the
four-domain time-series mixture SGEM is again the best label-free method
under the cluster gate (test NLL $0.883$).

Domain recovery, however, is the right success criterion only where the
domains differ sharply in difficulty: on the low-contrast time-series
mixture the partitions SGEM finds match the oracle domain partition in
real objective value while agreeing with it barely above chance. The
remaining gap to the oracle ceiling there ($0.789$ NLL) indicates that
matching the validation objective does not yet guarantee matching
test-time generalisation, which we regard as the most interesting open
question. Scaling the search space beyond the present $K \le 4$,
widening the seed sweep, and extending the framework to
transformer-based MoE layers remain natural next steps.

\section*{CRediT authorship contribution statement}
\textbf{Petr Babkin:} Conceptualization, Methodology, Software, Investigation, Writing -- original draft.
\textbf{Oleg Bakhteev:} Conceptualization, Methodology, Supervision, Writing -- review \& editing.

\section*{Declaration of competing interest}
The authors declare that they have no known competing financial interests or personal relationships that could have appeared to influence the work reported in this paper.

\section*{Data availability}
All datasets used in this work are either publicly available (CIFAR-100, SVHN, Monash Forecasting Archive) or synthetically generated by scripts included in the code repository at \coderepo. The source code of the proposed method and of all experiments is publicly available at the same repository.

\appendix
\renewcommand{\thesection}{\Alph{section}}
\renewcommand{\thesubsection}{\Alph{section}.\arabic{subsection}}

\section{Proof of Theorem~\ref{thm:sgem}}
\label{sec:proof}

This appendix collects the lemmas and the full proof of
Theorem~\ref{thm:sgem}. We follow the notation of
Section~\ref{subsec:theory}: $\theta = (\boldsymbol{\alpha}_{1:K}, \mathbf{R})
\in \Theta = \mathcal{A}^K \times (\Delta_K)^M$ with $\mathcal{A}$
finite, $\varepsilon_t$ is the per-iteration surrogate error
(in the loss scale; see~\eqref{eq:eps_t}), $N = \sum_m |\mathcal{C}_m|$,
and the variational free energy
parameterised by a vector-valued loss function $\boldsymbol{u}$ is
\begin{equation*}
\mathcal{F}_u(q, \theta) \;=\; \sum_{m=1}^{M} |\mathcal{C}_m|
\sum_{k=1}^{K} q_{mk}\, \log \frac{r_{mk}\, e^{-u_m(\boldsymbol{\alpha}_k, \mathbf{r}_k)}}{q_{mk}}.
\end{equation*}
We write $L_u(\theta) = \max_q \mathcal{F}_u(q, \theta)$ for the
size-weighted log-likelihood induced by loss $\boldsymbol{u}$, and abbreviate
$L_{\mathrm{true}} = L_{u_{\mathrm{true}}}$,
$L^{(t)} = L_{u^{(t)}}$. Every argument below is carried out row-wise in
$m$; the evaluation index $m$ is inert throughout, so the proof applies
verbatim to any surrogate satisfying Assumption~\ref{ass:bounded}.

\subsection{Preliminary Lemmas}

\begin{proposition}[ELBO maximisation]
\label{prop:elbo}
For every loss function
$\boldsymbol{u}: \mathcal{A} \times [0, 1]^M \to \mathbb{R}_{\geq 0}^{M}$
and every $\theta\in\Theta$,
$L_u(\theta) = \max_{q\in\mathcal{Q}} \mathcal{F}_u(q, \theta)$ with
the maximum attained at
\begin{equation*}
q^{\star}_{mk}(\theta) \;=\; \frac{r_{mk}\,e^{-u_m(\boldsymbol{\alpha}_k, \mathbf{r}_k)}}
{\sum_{j} r_{mj}\,e^{-u_m(\boldsymbol{\alpha}_j, \mathbf{r}_j)}}.
\end{equation*}
\end{proposition}

\begin{proof}
Fix $\theta$ and $m$. The factor $|\mathcal{C}_m|$ is a positive
constant; it can be pulled out of the inner maximisation in $q_{m,:}$
without affecting the optimum. Writing $a_{mk} = r_{mk}\,
e^{-u_m(\boldsymbol{\alpha}_k, \mathbf{r}_k)}$ and using
$\sum_k q_{mk} = 1$, the inner expression equals
$-\mathrm{KL}(q_{m,:} \| a_{m,:}/\!\sum_j a_{mj}) + \log\sum_j a_{mj}$.
The KL-divergence is non-negative and vanishes exactly at
$q_{mk} = a_{mk}/\sum_j a_{mj}$, which yields the claim. Summing
over $m$ with weights $|\mathcal{C}_m|$ recovers
$L_u(\theta) = \sum_m |\mathcal{C}_m|\,\log\sum_k a_{mk}$.
\end{proof}

A consequence is the Neal--Hinton decomposition
$L_u(\theta) = \mathcal{F}_u(q, \theta) + \sum_m |\mathcal{C}_m|\,
\mathrm{KL}\!\bigl(q_{m,:}\|q^{\star}_{m,:}(\theta)\bigr)$,
implying $L_u(\theta) \geq \mathcal{F}_u(q,\theta)$ with equality iff
$q = q^{\star}(\theta)$.

\begin{lemma}[Uniform ELBO deviation]
\label{lem:elbo-dev}
Under Assumption~\ref{ass:bounded}, for every $q$, $\theta$, $t$,
$\bigl|\mathcal{F}_{u_{\mathrm{true}}}(q,\theta) -
\mathcal{F}_{u^{(t)}}(q,\theta)\bigr| \leq N\,\varepsilon_t$.
\end{lemma}

\begin{proof}
The terms $\log r_{mk}$ and $-\log q_{mk}$ cancel between the two
ELBOs, leaving
\begin{equation*}
\mathcal{F}_{u_{\mathrm{true}}} - \mathcal{F}_{u^{(t)}}
\;=\; \sum_m |\mathcal{C}_m| \sum_k q_{mk}\,
\bigl[u^{(t)}_m(\boldsymbol{\alpha}_k, \mathbf{r}_k) -
u_{\mathrm{true}, m}(\boldsymbol{\alpha}_k, \mathbf{r}_k)\bigr].
\end{equation*}
Bounding the inner bracket by $\varepsilon_t$ --- the sup
in~\eqref{eq:eps_t} ranges over $m$ as well --- and using
$\sum_k q_{mk} = 1$ gives
$|\mathcal{F}_{u_{\mathrm{true}}} - \mathcal{F}_{u^{(t)}}|
\leq \varepsilon_t \sum_m |\mathcal{C}_m| = N\,\varepsilon_t$.
\end{proof}

\begin{corollary}
\label{cor:lik-dev}
For every $\theta$, $|L_{\mathrm{true}}(\theta) - L^{(t)}(\theta)|
\leq N\,\varepsilon_t$.
\end{corollary}

\begin{proof}
Apply Lemma~\ref{lem:elbo-dev} to the maxima of the two ELBOs and
use $|\max_q f - \max_q g| \leq \sup_q |f - g|$.
\end{proof}

Finally, we record three auxiliary bounds --- on the softmax map, on
the responsibilities, and on the Q-function --- used to analyse the
limit points in Part~(b) of the main proof; they compare the
surrogate-based and the true responsibilities.

\begin{lemma}[Softmax stability]
\label{lem:softmax}
Let $a_k, \tilde a_k > 0$ satisfy
$e^{-\varepsilon} a_k \leq \tilde a_k \leq e^{\varepsilon} a_k$. Then
the normalised vectors $p_k = a_k/\sum_j a_j$ and
$\tilde p_k = \tilde a_k/\sum_j \tilde a_j$ obey
$|\tilde p_k - p_k| \leq (e^{2\varepsilon} - 1)\,p_k \leq 3\varepsilon\,p_k$
for $\varepsilon \leq 1$.
\end{lemma}

\begin{proof}
$\tilde p_k \leq a_k e^{\varepsilon}/\sum_j a_j e^{-\varepsilon}
= p_k\,e^{2\varepsilon}$ and similarly $\tilde p_k \geq p_k\,e^{-2\varepsilon}$.
\end{proof}

\begin{lemma}[Responsibility deviation]
\label{lem:resp}
Let $w_{mk}(\theta) = q^{\star}_{mk}(\theta)$ be the
true-surrogate responsibilities (Proposition~\ref{prop:elbo} with
$\boldsymbol{u} = \boldsymbol{u}_{\mathrm{true}}$) and $q^{(t)}_{mk}$ the surrogate
responsibilities of the E-step at iteration $t$. For
$\varepsilon_t \leq 1$,
$|q^{(t)}_{mk} - w_{mk}(\theta^{(t)})| \leq 3\varepsilon_t\,
w_{mk}(\theta^{(t)})$.
\end{lemma}

\begin{proof}
Fix $m$ and apply Lemma~\ref{lem:softmax} with
$a_{mk} = r^{(t)}_{mk}\,e^{-u_{\mathrm{true}, m}(\boldsymbol{\alpha}^{(t)}_k, \mathbf{r}^{(t)}_k)}$
and $\tilde a_{mk}$ obtained by replacing $\boldsymbol{u}_{\mathrm{true}}$ by
$\boldsymbol{u}^{(t)}$. The ratio $\tilde a_{mk}/a_{mk}
= e^{u_{\mathrm{true}, m} - u^{(t)}_m}$ satisfies
$e^{-\varepsilon_t} \leq \tilde a_{mk}/a_{mk} \leq e^{\varepsilon_t}$
by~\eqref{eq:eps_t}, so Lemma~\ref{lem:softmax} applies with
$\varepsilon = \varepsilon_t$.
\end{proof}

\begin{lemma}[Q-function deviation]
\label{lem:qdev}
Define the surrogate Q-function
\begin{equation*}
\tilde Q^{(t)}(\theta;\theta') \;=\; \sum_{m=1}^M |\mathcal{C}_m|
\sum_{k=1}^K q^{(t)}_{mk}(\theta')
\bigl[\log r_{mk} \,-\, u^{(t)}_m(\boldsymbol{\alpha}_k, \mathbf{r}_k)\bigr].
\end{equation*}
Under Assumption~\ref{ass:bounded} there exists a constant
$C_Q = C_Q(N, K, U_{\max})$ such that for every $\theta, \theta'$ and
every $t$ with $\varepsilon_t \leq 1$,
$|\tilde Q^{(t)}(\theta;\theta') - Q_{\mathrm{true}}(\theta;\theta')|
\leq C_Q\,\varepsilon_t$.
\end{lemma}

\begin{proof}
Add and subtract
$\sum_m |\mathcal{C}_m| \sum_k q^{(t)}_{mk}\bigl[\log r_{mk}
- u_{\mathrm{true}, m}(\boldsymbol{\alpha}_k, \mathbf{r}_k)\bigr]$.
The first piece, the surrogate-error term, is
$\sum_m |\mathcal{C}_m| \sum_k q^{(t)}_{mk}
(u_{\mathrm{true}, m} - u^{(t)}_m)$ and is bounded by $N\varepsilon_t$.
The remainder splits as
$\sum_m |\mathcal{C}_m| \sum_k (q^{(t)}_{mk} - w_{mk})(-u_{\mathrm{true}, m})$
and
$\sum_m |\mathcal{C}_m| \sum_k (q^{(t)}_{mk} - w_{mk})\log r_{mk}$.
The first is bounded by $3NK\,U_{\max}\,\varepsilon_t$ via
Lemma~\ref{lem:resp} (because $|u_{\mathrm{true}, m}| \leq U_{\max}$
for every $m$).
For the second, using $|x\log x|\leq 1/e$ on $[0,1]$ together with
the softmax-stability bound of Lemma~\ref{lem:softmax}, one obtains a
bound of $3N\varepsilon_t/e$. Summing the three contributions gives
the claim with $C_Q = N(1 + 3K\,U_{\max} + 3/e)$.
\end{proof}

\subsection{Main Proof}

\begin{proposition}[Quasi-monotonicity of $L_{\mathrm{true}}$]
\label{prop:quasi-mono}
Under Assumptions~\ref{ass:bounded}--\ref{ass:gem}, for every $t$,
\begin{equation}
\label{eq:quasi-mono}
L_{\mathrm{true}}(\theta^{(t+1)}) \;\geq\;
L_{\mathrm{true}}(\theta^{(t)}) \;-\; 2N\,\varepsilon_t.
\end{equation}
\end{proposition}

\begin{proof}
We chain five inequalities. First,
$L_{\mathrm{true}}(\theta^{(t+1)}) \geq
\mathcal{F}_{u_{\mathrm{true}}}(q^{(t)}, \theta^{(t+1)})$ since
$L_u(\theta) = \max_q \mathcal{F}_u(q,\theta)$. Second, by
Lemma~\ref{lem:elbo-dev},
$\mathcal{F}_{u_{\mathrm{true}}}(q^{(t)}, \theta^{(t+1)})
\geq \mathcal{F}_{u^{(t)}}(q^{(t)}, \theta^{(t+1)}) - N\varepsilon_t$.
Third, the GEM Assumption~\ref{ass:gem} yields
$\mathcal{F}_{u^{(t)}}(q^{(t)}, \theta^{(t+1)})
\geq \mathcal{F}_{u^{(t)}}(q^{(t)}, \theta^{(t)})$. Fourth, the
E-step makes the latter equal to $L^{(t)}(\theta^{(t)})$. Fifth,
Corollary~\ref{cor:lik-dev} gives $L^{(t)}(\theta^{(t)}) \geq
L_{\mathrm{true}}(\theta^{(t)}) - N\varepsilon_t$. Combining the
five steps yields~\eqref{eq:quasi-mono}.
\end{proof}

\textbf{Part (a): convergence of the likelihood.}
Set $L_{\max} = \sup_{\theta\in\Theta} L_{\mathrm{true}}(\theta)$,
which is finite because $\Theta$ is compact (Assumption~\ref{ass:cont}
and finiteness of $\mathcal{A}$). Define
$V_t = L_{\max} - L_{\mathrm{true}}(\theta^{(t)}) \geq 0$. From
Proposition~\ref{prop:quasi-mono},
\begin{equation*}
V_{t+1} \;\leq\; V_t \;+\; 2N\varepsilon_t,
\qquad \sum_{t=0}^{\infty} 2N\varepsilon_t < \infty
\;\;\text{by Assumption~\ref{ass:summable}.}
\end{equation*}
The Robbins--Siegmund theorem~\cite{robbins1971convergence} (with
$a_t = 0$, $b_t = 2N\varepsilon_t$) implies that $\{V_t\}$ converges,
hence $\{L_{\mathrm{true}}(\theta^{(t)})\}$ converges. \qed

\textbf{Part (b): limit points are M-step fixed points.}

\textbf{Step 1: existence.} $\Theta$ is compact, so
$\{\theta^{(t)}\}$ has at least one limit point
$\theta^{\star} = \lim_{j\to\infty} \theta^{(t_j)}$.

\textbf{Step 2: contradiction setup.} Suppose the claim fails: there
are $\theta' \in \mathcal{M}^{\infty}(\theta^{\star})$ and $\delta > 0$
with
$Q_{\mathrm{true}}(\theta';\theta^{\star}) -
Q_{\mathrm{true}}(\theta^{\star};\theta^{\star}) = \delta > 0$.
Since $\theta'$ lies in $\mathcal{M}^{(t_j)}$ for infinitely many $j$,
we may pass to a further subsequence --- still written $(t_j)$ --- along
which $\theta^{(t_j)} \to \theta^{\star}$ \emph{and}
$\theta' \in \mathcal{M}^{(t_j)}$ for every $j$.

\textbf{Step 3: continuity.} The map $\theta'\mapsto Q_{\mathrm{true}}(\theta;\theta')$
is continuous on $\Theta$: $w_{mk}(\theta')$ is continuous because
each component $u_{\mathrm{true}, m}$ is continuous in $\mathbf{r}$
(Assumption~\ref{ass:cont}) and softmax is continuous; and the
product $w_{mk}\log r_{mk}$ extends continuously to $r_{mk} = 0$
because $w_{mk} = O(r_{mk})$ and $r\mapsto r\log r$ is continuous at
$0$. Thus, for $j$ large enough,
$Q_{\mathrm{true}}(\theta';\theta^{(t_j)}) -
Q_{\mathrm{true}}(\theta^{(t_j)};\theta^{(t_j)}) \geq \delta/2$.

\textbf{Step 4: transfer to $\tilde Q$.} By Lemma~\ref{lem:qdev},
\begin{equation*}
\tilde Q^{(t_j)}(\theta';\theta^{(t_j)}) -
\tilde Q^{(t_j)}(\theta^{(t_j)};\theta^{(t_j)}) \;\geq\;
\delta/2 - 2C_Q\,\varepsilon_{t_j} \;\geq\; \delta/4
\end{equation*}
for $j$ large enough, since $\varepsilon_{t_j}\to 0$.

\textbf{Step 5: M-step ascent.} By Assumption~\ref{ass:gem} the iterate
$\theta^{(t_j+1)}$ maximises
$\mathcal{F}_{u^{(t_j)}}(q^{(t_j)}, \cdot)$ over
$\mathcal{M}^{(t_j)}$, and $\theta' \in \mathcal{M}^{(t_j)}$ by
Step~2. The responsibilities $q^{(t_j)}$ are held fixed in that
maximisation, and expanding the logarithm in $\mathcal{F}$ gives
\begin{equation*}
\mathcal{F}_{u^{(t)}}\bigl(q^{(t)}, \theta\bigr)
\;=\; \tilde Q^{(t)}\bigl(\theta;\theta^{(t)}\bigr)
\;-\; \sum_{m=1}^{M} |\mathcal{C}_m| \sum_{k=1}^{K}
q^{(t)}_{mk}\,\log q^{(t)}_{mk},
\end{equation*}
whose second term does not depend on $\theta$. The two objectives
therefore have the same maximisers over any candidate set, and
$\tilde Q^{(t_j)}(\theta^{(t_j+1)};\theta^{(t_j)}) \geq
\tilde Q^{(t_j)}(\theta';\theta^{(t_j)})$. Translating back through
Lemma~\ref{lem:qdev} once more,
\begin{equation*}
\Delta_{t_j} := Q_{\mathrm{true}}(\theta^{(t_j+1)};\theta^{(t_j)}) -
Q_{\mathrm{true}}(\theta^{(t_j)};\theta^{(t_j)}) \;\geq\;
\delta/4 - 2C_Q\,\varepsilon_{t_j} \;\geq\; \delta/8
\end{equation*}
for $j$ large enough.

\textbf{Step 6: summability contradiction.} The Gibbs inequality
$L_{\mathrm{true}}(\theta^{(t+1)}) - L_{\mathrm{true}}(\theta^{(t)})
\geq \Delta_t$ together with Proposition~\ref{prop:quasi-mono}
implies $\Delta_t \geq -2N\varepsilon_t$. Hence the non-negative
sequence $\Delta_t + 2N\varepsilon_t$ satisfies
$\sum_t (\Delta_t + 2N\varepsilon_t) < \infty$ by
Robbins--Siegmund applied to $V_t$, and consequently
$\sum_t \Delta_t < \infty$. But Step~5 gives $\Delta_{t_j} \geq
\delta/8$ for infinitely many $j$, yielding $\sum_t \Delta_t = +\infty$
--- a contradiction. Therefore no
$\theta' \in \mathcal{M}^{\infty}(\theta^{\star})$ improves
$Q_{\mathrm{true}}$ at $\theta^{\star}$, which is the claim. \qed

\subsection{Convergence Rate}

A by-product of the proof of Theorem~\ref{thm:sgem}(a) is the
explicit rate
\begin{equation}
L_{\mathrm{true}}(\theta^{\star}) - L_{\mathrm{true}}(\theta^{(T)})
\;\leq\; 2N \sum_{t = T}^{\infty} \varepsilon_t.
\end{equation}
If $\varepsilon_t = O(t^{-p})$ with $p > 1$, this gives
$L_{\mathrm{true}}(\theta^{(T)}) \to L_{\mathrm{true}}(\theta^{\star})$
at rate $O(T^{1 - p})$, matching the rate observed for stochastic-EM
analyses~\cite{cappe2009online}.
 
\section{Exploration Mechanisms in SGEM}
\label{app:exploration}

This appendix details the three exploration mechanisms introduced in
Section~\ref{subsec:sgem}. All of them consume only quantities the
search has already produced --- the accumulated surrogate observations
$\mathcal{D}_{\text{surr}}$ --- and none of them consults domain labels.

\textbf{Cluster affinities.} The observations collected so far form a
cluster$\times$architecture matrix of per-cluster validation losses,
whose $(m, i)$ entry is the loss $y^{(i)}_m$ attained on cluster $m$ by
the $i$-th trained architecture. Subtracting the column mean removes the
global effect of an architecture, so that what remains is how a cluster
\emph{reacts} to architectures rather than how hard the cluster is in
absolute terms --- without this centring the matrix is dominated by the
overall difficulty ordering of the clusters, which carries no
architectural information. Clusters with correlated residual profiles
are then grouped by agglomerative clustering; the resulting groups drive
the first two mechanisms below.

\textbf{Exploration masks.} Each affinity group is proposed as a
training mask $\boldsymbol{b}$ and evaluated with several architectures,
the resulting triples entering $\mathcal{D}_{\text{surr}}$ as ordinary
observations. A group is by construction much smaller than a current
routing region, so these are precisely the near-pure masks that the LCB
acquisition of Section~\ref{subsec:surr} never proposes on its own: that
acquisition scores candidates under the current routing and therefore
keeps drawing masks close to it, leaving the surrogate without evidence
on what a specialised expert would attain. The mechanism is what breaks
the feedback loop in which a surrogate trained only around the incumbent
split keeps confirming that split.

\textbf{Partition jumps.} The same affinities also propose complete
candidate partitions of the $M$ clusters into $K$ regions. A candidate
partition is evaluated by genuinely training, for each of its regions,
the incumbent architecture of the corresponding expert and one further
architecture selected by the surrogate; it replaces the current routing
only if the resulting \emph{measured} objective exceeds that of the
incumbent partition. Acceptance is therefore decided on real training
outcomes rather than on surrogate predictions, so a mis-calibrated
surrogate can waste evaluation budget but cannot move the routing to a
worse partition. The mechanism matters when the routing sits in a local
optimum that no single-cluster move can leave, which the gradient M-step
on $\mathbf{R}$ alone cannot escape.

\textbf{Final per-region refit.} After the EM loop terminates, the hard
assignment $I_k$ of Section~\ref{subsec:searchspace} is fixed and each
expert's architecture is re-searched from scratch by random search
restricted to its own region $\mathcal{D}_{\mathbf{r}_k}$. At this point
no routing variable is left and the problem is an ordinary NAS problem
per region.
 
\section{Additional Experimental Results}
\label{app:experiments}

This appendix collects two ablations and a convergence trace on the
CIFAR-100/SVHN mixture
($K=2$, $M=30$ semantic clusters) referenced from
Section~\ref{subsec:cifar}. All numbers are cluster-gate test accuracy
(mean $\pm$ std over the indicated number of seeds), evaluated under the
same protocol as Table~\ref{tab:cifar_main}. Both ablations were run
with a simplified surrogate that predicts a single scalar per candidate
--- the region average of~\eqref{eq:scalar_u} --- instead of the
per-cluster vector of Definition~\ref{def:surrogate}; this reduction
requires the load-balancing term discussed in
Appendix~\ref{app:lambda} (used here at $\lambda_{\text{LB}} = 1$) and
costs about $2$~p.p.\ of accuracy overall, so the levels below sit
under the $0.6402 \pm 0.002$ of Table~\ref{tab:cifar_main}. The
comparison \emph{within} each block is unaffected.

\subsection{Decomposition of the cluster-gate accuracy}
\label{app:decomp}

To attribute the MoE gain to its two ingredients --- the cluster$\to$expert
\emph{routing} and the per-expert \emph{architecture} --- we vary one
factor at a time (Table~\ref{tab:decomp}). Block~(i) fixes a single
DARTS architecture (searched on the full mixture) and only changes the
routing: moving from a random cluster split to the oracle source split
raises accuracy by $3.0$ percentage points, so the routing is the main
lever. Block~(ii) fixes the oracle routing and instead changes the
architecture: a single full-mix architecture, a CIFAR-searched one and an
SVHN-searched one all lie within $1$~p.p.\ of each other, i.e.\ once the
routing is correct, architectural heterogeneity between experts adds
little on this benchmark. Block~(iii) isolates the contribution of
\emph{learning} the routing inside SGEM: freezing the routing matrix at
its random initialisation (no M-step on $\mathbf{R}$) drops the
source-recovery rate to chance ($59\%$) and the accuracy to $0.598$,
whereas the full method recovers the source partition ($92\%$) and
reaches $0.620$ --- essentially the value obtained with the oracle
routing in block~(i), but \emph{without} any source labels.

\begin{table}[h]
\centering
\caption{Ablation decomposing the cluster-gate accuracy. ``src.\ match''
is the agreement of the hard routing with the ground-truth source
partition; it is $100\%$ by construction whenever the oracle split is
used.}
\label{tab:decomp}
\begin{tabular}{lcc}
\toprule
Setting & test acc. & src.\ match \\
\midrule
\multicolumn{3}{l}{\textit{(i) vary routing, architecture fixed (full-mix DARTS)}} \\
\quad random cluster split & $0.587 \pm 0.006$ & $\approx 54\%$ \\
\quad oracle source split & $0.617 \pm 0.007$ & $100\%$ \\
\midrule
\multicolumn{3}{l}{\textit{(ii) vary architecture, oracle routing}} \\
\quad full-mix architecture & $0.627 \pm 0.009$ & $100\%$ \\
\quad CIFAR-searched architecture & $0.636 \pm 0.003$ & $100\%$ \\
\quad SVHN-searched architecture & $0.626 \pm 0.006$ & $100\%$ \\
\midrule
\multicolumn{3}{l}{\textit{(iii) learn the routing (SGEM, no labels)}} \\
\quad routing frozen at random init & $0.598 \pm 0.003$ & $59 \pm 7\%$ \\
\quad \textbf{SGEM (full)} & $\mathbf{0.620 \pm 0.009}$ & $\mathbf{92 \pm 2\%}$ \\
\bottomrule
\end{tabular}
\end{table}

\subsection{The scalar surrogate and load balancing}
\label{app:lambda}

\textbf{Why the scalar reduction needs a penalty.} The ablations above
approximate the conditional factor in~\eqref{eq:moe_persample_lik} by a
quantity that does not depend on the evaluation cluster, replacing
$u_m(\boldsymbol{\alpha}_k, \mathbf{r}_k)$ by the region average
$\bar u(\boldsymbol{\alpha}_k, \mathbf{r}_k)$ of~\eqref{eq:scalar_u}.
This makes the routing sub-problem degenerate. Fix the architectures and
treat the values $\bar u_k = \bar u(\boldsymbol{\alpha}_k, \mathbf{r}_k)$
as constants: each row of~\eqref{eq:moe_objective} is then a linear form
$\sum_k r_{mk} e^{-\bar u_k}$ on the simplex, whose maximiser is the
one-hot vector at $\arg\min_k \bar u_k$ --- \emph{the same expert for
every cluster}. The objective is maximised by routing all clusters to a
single expert, and specialisation is opposed only by the implicit
dependence of $\bar u_k$ on $\mathbf{r}_k$, so a scalar surrogate
collapses unless its M-step is modified. The per-cluster surrogate of
Definition~\ref{def:surrogate}, which the method uses throughout, is free
of this effect: the row-wise maximiser
$\arg\min_k u_m(\boldsymbol{\alpha}_k, \mathbf{r}_k)$ varies with $m$, and
specialisation is the argmax of~\eqref{eq:moe_objective} itself.

The collapse is a property of the scalar reduction of the objective
rather than of the optimiser, and it is visible directly in
the EM dynamics of Algorithm~\ref{alg:sgem}. At a hard routing, where every
cluster $m$ is sent to a single expert $k(m)$ (so $r_{mk} = [\,k = k(m)\,]$),
the inner sum of~\eqref{eq:moe_objective} has a single non-zero term and
the logarithm cancels the exponential, giving
\begin{equation}
\label{eq:hard_objective}
\mathcal{J} \;=\; \sum_{m=1}^{M} |\mathcal{C}_m|\,
\bigl(-\bar u(\boldsymbol{\alpha}_{k(m)}, \mathbf{r}_{k(m)})\bigr)
\;=\; -\sum_{k=1}^{K} N_k\, \bar u_k,
\qquad N_k = \!\!\sum_{m:\,k(m)=k}\!\! |\mathcal{C}_m|,
\end{equation}
with $\bar u_k = \bar u(\boldsymbol{\alpha}_k, \mathbf{r}_k)$ the predicted loss of
expert $k$ and $N_k$ the number of samples routed to it. If the surrogate
becomes nearly constant in its routing argument and the experts are
interchangeable ($\bar u_k \equiv \bar u$), then $\mathcal{J} = -\bar u\, N$ is
independent of the assignment $\{N_k\}$, so the likelihood is flat over
the routing simplex; and even when the $\bar u_k$ differ, $\mathcal{J}$ is
linear in the counts $N_k$ and is maximised by routing all clusters to
the single expert with the smallest $\bar u_k$. Under the per-cluster
surrogate the same hard routing gives
$\mathcal{J} = -\sum_k \sum_{m \in I_k} |\mathcal{C}_m|\, u_m(\boldsymbol{\alpha}_k, \mathbf{r}_k)$,
which is no longer linear in the counts and whose per-cluster maximiser
varies with $m$; the penalty is then unnecessary. To counteract the
collapse in the scalar case we augment
the M-step of Algorithm~\ref{alg:sgem} with a load-balancing penalty
\begin{equation}
\label{eq:loadbalance}
\mathrm{LB}(\mathbf{R}) \;=\; K \cdot \sum_{k=1}^{K}\bar{r}_k \cdot p_k,
\qquad
\bar{r}_k = \frac{1}{M}\sum_{m=1}^{M} r_{mk}, \quad
p_k = \frac{|\{m : \arg\max_{k'} r_{mk'} = k\}|}{M},
\end{equation}
weighted by a coefficient $\lambda_{\text{LB}}$; it is minimal for a
uniform routing and saturates at $K$ under full collapse.

\textbf{Choosing the weight.} We keep the weight of these
runs small ($\lambda_{\text{LB}}=1$) because the correct
source partition is itself imbalanced ($19/11$), so a strong balancing
penalty would push the routing away from it; collapse is prevented
instead through the data-collection scheme of
Section~\ref{subsec:cifar}. We make this trade-off quantitative by scoring three
candidate hard routings under the \emph{penalised} objective
$L - \lambda_{\text{LB}}\cdot\mathrm{LB}_N$ actually optimised in the
M-step, where the per-expert losses $u_k$ are obtained by genuinely
training a fixed architecture on each expert's clusters (no surrogate),
$L = -\sum_k N_k u_k$ is the resulting size-weighted objective, and
$\mathrm{LB}_N = K\sum_k P_k^2\cdot N$ is the (extensive) balancing
penalty (Table~\ref{tab:lambda}).

\begin{table}[h]
\centering
\caption{Penalised objective of three hard routings on real-trained
per-expert losses. The penalised score is
$L - \lambda_{\text{LB}}\,\mathrm{LB}_N$.}
\label{tab:lambda}
\begin{tabular}{lccc}
\toprule
hard split & $u_k$ (per expert) & $L$ & $\mathrm{LB}_N$ \\
\midrule
oracle $[19, 11]$ & $[2.73,\ 0.51]$ & $-48099$ & $31491$ \\
balanced, SGEM $[15, 15]$ & $[1.91,\ 1.90]$ & $-55973$ & $29400$ \\
balanced, random $[15, 15]$ & $[1.77,\ 2.15]$ & $-57726$ & $29400$ \\
\bottomrule
\end{tabular}
\end{table}

On the unpenalised objective ($\lambda_{\text{LB}}=0$) the oracle split is
preferred by a margin of $\Delta L = 7874$. However, because the oracle
split is imbalanced it is penalised more heavily than a balanced one by
$\Delta\,\mathrm{LB}_N = 2091$ per unit of $\lambda_{\text{LB}}$. The
oracle split therefore wins the penalised objective only while
$\lambda_{\text{LB}} < \Delta L / \Delta\,\mathrm{LB}_N \approx 3.77$;
beyond this cross-over a balanced (and hence domain-blind) split scores
higher. This is a direct, surrogate-free conflict between load balancing
and the imbalanced domain structure, and explains why a large
$\lambda_{\text{LB}}$ --- which reliably prevents collapse --- would also
suppress the very specialisation we seek; we therefore operate well below
the cross-over.

Two observations sharpen this. First, the penalty~\eqref{eq:loadbalance}
measures balance in \emph{clusters}: $P_k$ is the fraction of clusters
routed to expert $k$. The oracle split is imbalanced in clusters
($19/11$) but, since the benchmark mixes the two sources in equal
proportion, almost perfectly balanced in \emph{samples}. Replacing $P_k$
by the sample fraction $N_k / N$ therefore leaves the oracle split
essentially unpenalised and dissolves the cross-over of
Table~\ref{tab:lambda}. Second, and more fundamentally, the penalty is
needed at all only under the scalar reduction; the method itself uses
the per-cluster surrogate throughout, drops the term entirely and shows
no collapse, as the runs of Table~\ref{tab:cifar_main} confirm
($\lambda_{\text{LB}} = 0$, five seeds). The M-step of
Algorithm~\ref{alg:sgem} is therefore run exactly as analysed in
Section~\ref{subsec:theory}.

\textbf{What the scalar reduction costs.} Run under the same data
collection as the headline results, the scalar-surrogate variant reaches
$0.6199 \pm 0.009$ cluster-gate test accuracy over three seeds, with a
source-recovery rate of $92\% \pm 2\%$, against the $0.6402 \pm 0.002$
and $95\% \pm 2\%$ of Table~\ref{tab:cifar_main}. The $2.0$~p.p.\ gap is
paid despite the penalty above, which prevents the collapse but cannot
restore the per-cluster signal the routing needs --- which is why the
surrogate is vector-valued in the first place.

\subsection{Convergence of a single run}
\label{app:convergence}

Tracing one run of the scalar-surrogate variant across its EM iterations
(seed $322$), the size-weighted objective $L^{(t)}$ increases up to
optimisation noise while the source-recovery rate climbs from $67\%$ to
$100\%$, which is the behaviour Theorem~\ref{thm:sgem} predicts. The two
curves peak one iteration apart --- the objective at $t = 8$, the
recovery at $t = 9$ --- the same objective/quality mismatch near the
optimum discussed in Section~\ref{subsec:cifar}, so selecting the
iterate by best objective returns a near-ideal but not always perfect
split.

\bibliographystyle{elsarticle-num}

\begin{thebibliography}{10}
\expandafter\ifx\csname url\endcsname\relax
  \def\url#1{\texttt{#1}}\fi
\expandafter\ifx\csname urlprefix\endcsname\relax\def\urlprefix{URL }\fi
\expandafter\ifx\csname href\endcsname\relax
  \def\href#1#2{#2} \def\path#1{#1}\fi

\bibitem{jacobs1991adaptive}
R.~A. Jacobs, M.~I. Jordan, S.~J. Nowlan, G.~E. Hinton, Adaptive mixtures of
  local experts, Neural computation 3~(1) (1991) 79--87.
\newblock \href {https://doi.org/10.1162/neco.1991.3.1.79}
  {\path{doi:10.1162/neco.1991.3.1.79}}.

\bibitem{mu2025comprehensive}
S.~Mu, S.~Lin, A comprehensive survey of mixture-of-experts: Algorithms,
  theory, and applications, arXiv preprint arXiv:2503.07137 (2025).
\newblock \href {https://doi.org/10.48550/arXiv.2503.07137}
  {\path{doi:10.48550/arXiv.2503.07137}}.

\bibitem{fedus2022switch}
W.~Fedus, B.~Zoph, N.~Shazeer, Switch transformers: Scaling to trillion
  parameter models with simple and efficient sparsity, Journal of Machine
  Learning Research 23~(120) (2022) 1--39.
\newblock \href {https://doi.org/10.48550/arXiv.2101.03961}
  {\path{doi:10.48550/arXiv.2101.03961}}.

\bibitem{jiang2024mixtral}
A.~Q. Jiang, A.~Sablayrolles, A.~Roux, A.~Mensch, B.~Savary, C.~Bamford, D.~S.
  Chaplot, D.~d.~l. Casas, E.~B. Hanna, F.~Bressand, et~al., Mixtral of
  experts, arXiv preprint arXiv:2401.04088 (2024).
\newblock \href {https://doi.org/10.48550/arXiv.2401.04088}
  {\path{doi:10.48550/arXiv.2401.04088}}.

\bibitem{ren2021comprehensive}
P.~Ren, Y.~Xiao, X.~Chang, P.-Y. Huang, Z.~Li, X.~Chen, X.~Wang, A
  comprehensive survey of neural architecture search: Challenges and solutions,
  ACM Computing Surveys (CSUR) 54~(4) (2021) 1--34.
\newblock \href {https://doi.org/10.1145/3447582} {\path{doi:10.1145/3447582}}.

\bibitem{liu2018darts}
H.~Liu, K.~Simonyan, Y.~Yang, Darts: Differentiable architecture search, arXiv
  preprint arXiv:1806.09055 (2018).
\newblock \href {https://doi.org/10.48550/arXiv.1806.09055}
  {\path{doi:10.48550/arXiv.1806.09055}}.

\bibitem{chen2022towards}
Z.~Chen, Y.~Deng, Y.~Wu, Q.~Gu, Y.~Li, Towards understanding the
  mixture-of-experts layer in deep learning, Advances in neural information
  processing systems 35 (2022) 23049--23062.
\newblock \href {https://doi.org/10.52202/068431-1675}
  {\path{doi:10.52202/068431-1675}}.

\bibitem{chai2022mixture}
T.~Chai~Fung, S.~C. Tseung, Mixture of experts models for multilevel data:
  modelling framework and approximation theory, arXiv e-prints (2022)
  arXiv--2209\href {https://doi.org/10.48550/arXiv.2209.15207}
  {\path{doi:10.48550/arXiv.2209.15207}}.

\bibitem{guo2018multi}
J.~Guo, D.~J. Shah, R.~Barzilay, Multi-source domain adaptation with mixture of
  experts, arXiv preprint arXiv:1809.02256 (2018).
\newblock \href {https://doi.org/10.48550/arXiv.1809.02256}
  {\path{doi:10.48550/arXiv.1809.02256}}.

\bibitem{hu2025multi}
Z.~Hu, V.~Guti{\'e}rrez-Basulto, Z.~Xiang, R.~Li, J.~Z. Pan, Multi-level
  mixture of experts for multimodal entity linking, in: Proceedings of the 31st
  ACM SIGKDD Conference on Knowledge Discovery and Data Mining V. 2, 2025, pp.
  979--990.
\newblock \href {https://doi.org/10.1145/3711896.3737060}
  {\path{doi:10.1145/3711896.3737060}}.

\bibitem{jordan1994hierarchical}
M.~I. Jordan, R.~A. Jacobs, Hierarchical mixtures of experts and the em
  algorithm, Neural computation 6~(2) (1994) 181--214.
\newblock \href {https://doi.org/10.1162/neco.1994.6.2.181}
  {\path{doi:10.1162/neco.1994.6.2.181}}.

\bibitem{jacobs1997bayesian}
R.~A. Jacobs, F.~Peng, M.~A. Tanner, A bayesian approach to model selection in
  hierarchical mixtures-of-experts architectures, Neural Networks 10~(2) (1997)
  231--241.
\newblock \href {https://doi.org/10.1016/s0893-6080(96)00050-0}
  {\path{doi:10.1016/s0893-6080(96)00050-0}}.

\bibitem{nguyen2023general}
H.~Nguyen, P.~Akbarian, T.~Nguyen, N.~Ho, A general theory for softmax gating
  multinomial logistic mixture of experts, arXiv preprint arXiv:2310.14188
  (2023).
\newblock \href {https://doi.org/10.48550/arXiv.2310.14188}
  {\path{doi:10.48550/arXiv.2310.14188}}.

\bibitem{nguyen2023demystifying}
H.~Nguyen, T.~Nguyen, N.~Ho, Demystifying softmax gating function in gaussian
  mixture of experts, Advances in Neural Information Processing Systems 36
  (2023) 4624--4652.
\newblock \href {https://doi.org/10.52202/075280-0206}
  {\path{doi:10.52202/075280-0206}}.

\bibitem{piwko2025divide}
J.~Piwko, J.~Ruci{\'n}ski, D.~P{\l}udowski, A.~Zajko, P.~{\.Z}ak,
  M.~Zacharecki, A.~Kozak, K.~Wo{\'z}nica, Divide, specialize, and route: A new
  approach to efficient ensemble learning, arXiv preprint arXiv:2506.20814
  (2025).
\newblock \href {https://doi.org/10.48550/arXiv.2506.20814}
  {\path{doi:10.48550/arXiv.2506.20814}}.

\bibitem{rasmussen2001infinite}
C.~Rasmussen, Z.~Ghahramani, Infinite mixtures of gaussian process experts,
  Advances in neural information processing systems 14 (2001).

\bibitem{eigen2013learning}
D.~Eigen, M.~Ranzato, I.~Sutskever, Learning factored representations in a deep
  mixture of experts, arXiv preprint arXiv:1312.4314 (2013).
\newblock \href {https://doi.org/10.48550/arXiv.1312.4314}
  {\path{doi:10.48550/arXiv.1312.4314}}.

\bibitem{shazeer2017outrageously}
N.~Shazeer, A.~Mirhoseini, K.~Maziarz, A.~Davis, Q.~Le, G.~Hinton, J.~Dean,
  Outrageously large neural networks: The sparsely-gated mixture-of-experts
  layer, arXiv preprint arXiv:1701.06538 (2017).
\newblock \href {https://doi.org/10.48550/arXiv.1701.06538}
  {\path{doi:10.48550/arXiv.1701.06538}}.

\bibitem{riquelme2021scaling}
C.~Riquelme, J.~Puigcerver, B.~Mustafa, M.~Neumann, R.~Jenatton,
  A.~Susano~Pinto, D.~Keysers, N.~Houlsby, Scaling vision with sparse mixture
  of experts, Advances in Neural Information Processing Systems 34 (2021)
  8583--8595.
\newblock \href {https://doi.org/10.48550/arXiv.2106.05974}
  {\path{doi:10.48550/arXiv.2106.05974}}.

\bibitem{dai2024deepseekmoe}
D.~Dai, C.~Deng, C.~Zhao, R.~Xu, H.~Gao, D.~Chen, J.~Li, W.~Zeng, X.~Yu, Y.~Wu,
  et~al., Deepseekmoe: Towards ultimate expert specialization in
  mixture-of-experts language models, arXiv preprint arXiv:2401.06066 (2024).
\newblock \href {https://doi.org/10.48550/arXiv.2401.06066}
  {\path{doi:10.48550/arXiv.2401.06066}}.

\bibitem{qiu2025demons}
Z.~Qiu, Z.~Huang, B.~Zheng, K.~Wen, Z.~Wang, R.~Men, I.~Titov, D.~Liu, J.~Zhou,
  J.~Lin, Demons in the detail: On implementing load balancing loss for
  training specialized mixture-of-expert models, arXiv preprint
  arXiv:2501.11873 (2025).
\newblock \href {https://doi.org/10.48550/arXiv.2501.11873}
  {\path{doi:10.48550/arXiv.2501.11873}}.

\bibitem{gururangan2022demix}
S.~Gururangan, M.~Lewis, A.~Holtzman, N.~A. Smith, L.~Zettlemoyer, {DEMix}
  layers: Disentangling domains for modular language modeling, in: Proceedings
  of the 2022 Conference of the North American Chapter of the Association for
  Computational Linguistics: Human Language Technologies, 2022, pp. 5557--5576.
\newblock \href {https://doi.org/10.18653/v1/2022.naacl-main.407}
  {\path{doi:10.18653/v1/2022.naacl-main.407}}.

\bibitem{li2022branch}
M.~Li, S.~Gururangan, T.~Dettmers, M.~Lewis, T.~Althoff, N.~A. Smith,
  L.~Zettlemoyer, Branch-train-merge: Embarrassingly parallel training of
  expert language models, arXiv preprint arXiv:2208.03306 (2022).
\newblock \href {https://doi.org/10.48550/arXiv.2208.03306}
  {\path{doi:10.48550/arXiv.2208.03306}}.

\bibitem{gururangan2023scaling}
S.~Gururangan, M.~Li, M.~Lewis, W.~Shi, T.~Althoff, N.~A. Smith,
  L.~Zettlemoyer, Scaling expert language models with unsupervised domain
  discovery, arXiv preprint arXiv:2303.14177 (2023).
\newblock \href {https://doi.org/10.48550/arXiv.2303.14177}
  {\path{doi:10.48550/arXiv.2303.14177}}.

\bibitem{sukhbaatar2024branch}
S.~Sukhbaatar, O.~Golovneva, V.~Sharma, H.~Xu, X.~V. Lin, B.~Rozi{\`e}re,
  J.~Kahn, D.~Li, W.-t. Yih, J.~Weston, et~al., Branch-train-mix: Mixing expert
  {LLMs} into a mixture-of-experts {LLM}, in: First Conference on Language
  Modeling, 2024.
\newblock \href {https://doi.org/10.48550/arXiv.2403.07816}
  {\path{doi:10.48550/arXiv.2403.07816}}.

\bibitem{muennighoff2025olmoe}
N.~Muennighoff, L.~Soldaini, D.~Groeneveld, K.~Lo, J.~Morrison, S.~Min, W.~Shi,
  P.~Walsh, O.~Tafjord, N.~Lambert, et~al., {OLMoE}: Open mixture-of-experts
  language models, in: The Thirteenth International Conference on Learning
  Representations, 2025.
\newblock \href {https://doi.org/10.48550/arXiv.2409.02060}
  {\path{doi:10.48550/arXiv.2409.02060}}.

\bibitem{shi2025timemoe}
X.~Shi, S.~Wang, Y.~Nie, D.~Li, Z.~Ye, Q.~Wen, M.~Jin, Time-{MoE}:
  Billion-scale time series foundation models with mixture of experts, in: The
  Thirteenth International Conference on Learning Representations, 2025.
\newblock \href {https://doi.org/10.48550/arXiv.2409.16040}
  {\path{doi:10.48550/arXiv.2409.16040}}.

\bibitem{liu2025moiraimoe}
X.~Liu, J.~Liu, G.~Woo, T.~Aksu, Y.~Liang, R.~Zimmermann, C.~Xiong,
  S.~Savarese, D.~Sahoo, Moirai-{MoE}: Empowering time series foundation models
  with sparse mixture of experts, in: Forty-second International Conference on
  Machine Learning, 2025.
\newblock \href {https://doi.org/10.48550/arXiv.2410.10469}
  {\path{doi:10.48550/arXiv.2410.10469}}.

\bibitem{zoph2016neural}
B.~Zoph, Q.~V. Le, Neural architecture search with reinforcement learning,
  arXiv preprint arXiv:1611.01578 (2016).
\newblock \href {https://doi.org/10.48550/arXiv.1611.01578}
  {\path{doi:10.48550/arXiv.1611.01578}}.

\bibitem{baker2016designing}
B.~Baker, O.~Gupta, N.~Naik, R.~Raskar, Designing neural network architectures
  using reinforcement learning, arXiv preprint arXiv:1611.02167 (2016).
\newblock \href {https://doi.org/10.48550/arXiv.1611.02167}
  {\path{doi:10.48550/arXiv.1611.02167}}.

\bibitem{zoph2018learning}
B.~Zoph, V.~Vasudevan, J.~Shlens, Q.~V. Le, Learning transferable architectures
  for scalable image recognition, in: Proceedings of the IEEE conference on
  computer vision and pattern recognition, 2018, pp. 8697--8710.
\newblock \href {https://doi.org/10.1109/cvpr.2018.00907}
  {\path{doi:10.1109/cvpr.2018.00907}}.

\bibitem{real2019regularized}
E.~Real, A.~Aggarwal, Y.~Huang, Q.~V. Le, Regularized evolution for image
  classifier architecture search, in: Proceedings of the aaai conference on
  artificial intelligence, Vol.~33, 2019, pp. 4780--4789.
\newblock \href {https://doi.org/10.1609/aaai.v33i01.33014780}
  {\path{doi:10.1609/aaai.v33i01.33014780}}.

\bibitem{yang2020cars}
Z.~Yang, Y.~Wang, X.~Chen, B.~Shi, C.~Xu, C.~Xu, Q.~Tian, C.~Xu, Cars:
  Continuous evolution for efficient neural architecture search, in:
  Proceedings of the IEEE/CVF conference on computer vision and pattern
  recognition, 2020, pp. 1829--1838.
\newblock \href {https://doi.org/10.1109/cvpr42600.2020.00190}
  {\path{doi:10.1109/cvpr42600.2020.00190}}.

\bibitem{chen2020stabilizing}
X.~Chen, C.-J. Hsieh, Stabilizing differentiable architecture search via
  perturbation-based regularization, in: International conference on machine
  learning, PMLR, 2020, pp. 1554--1565.
\newblock \href {https://doi.org/10.48550/arXiv.2002.05283}
  {\path{doi:10.48550/arXiv.2002.05283}}.

\bibitem{you2020greedynas}
S.~You, T.~Huang, M.~Yang, F.~Wang, C.~Qian, C.~Zhang, Greedynas: Towards fast
  one-shot nas with greedy supernet, in: Proceedings of the IEEE/CVF Conference
  on Computer Vision and Pattern Recognition, 2020, pp. 1999--2008.
\newblock \href {https://doi.org/10.1109/CVPR42600.2020.00207}
  {\path{doi:10.1109/CVPR42600.2020.00207}}.

\bibitem{nayman2019xnas}
N.~Nayman, A.~Noy, T.~Ridnik, I.~Friedman, R.~Jin, L.~Zelnik, Xnas: Neural
  architecture search with expert advice, Advances in neural information
  processing systems 32 (2019).
\newblock \href {https://doi.org/10.48550/arXiv.1906.08031}
  {\path{doi:10.48550/arXiv.1906.08031}}.

\bibitem{jawahar2022automoe}
G.~Jawahar, S.~Mukherjee, X.~Liu, Y.~J. Kim, M.~Abdul-Mageed, L.~V. Lakshmanan,
  A.~H. Awadallah, S.~Bubeck, J.~Gao, Automoe: Heterogeneous mixture-of-experts
  with adaptive computation for efficient neural machine translation, arXiv
  preprint arXiv:2210.07535 (2022).
\newblock \href {https://doi.org/10.48550/arXiv.2210.07535}
  {\path{doi:10.48550/arXiv.2210.07535}}.

\bibitem{zhou2023brainformers}
Y.~Zhou, N.~Du, Y.~Huang, D.~Peng, C.~Lan, D.~Huang, S.~Shakeri, D.~So, A.~Dai,
  Y.~Lu, Z.~Chen, Q.~Le, C.~Cui, J.~Laudon, J.~Dean, Brainformers: Trading
  simplicity for efficiency, arXiv preprint arXiv:2306.00008 (2023).
\newblock \href {https://doi.org/10.48550/arXiv.2306.00008}
  {\path{doi:10.48550/arXiv.2306.00008}}.

\bibitem{lukhi2026heterogeneous}
Y.~R. Lukhi, H.~R. Moradiya, R.~Timofte, D.~Ignatov, Systematic exploration of
  4-expert heterogeneous mixture-of-experts via automated pipeline search,
  arXiv preprint arXiv:2606.23739 (2026).
\newblock \href {https://doi.org/10.48550/arXiv.2606.23739}
  {\path{doi:10.48550/arXiv.2606.23739}}.

\bibitem{han2024cmn}
S.~Han, S.~Liu, S.~Du, M.~Li, Z.~Ye, X.~Xu, Y.~Li, Z.~Wang, D.~Shang, Cmn: a
  co-designed neural architecture search for efficient
  computing-in-memory-based mixture-of-experts, Science China Information
  Sciences 67~(10) (2024) 200405.
\newblock \href {https://doi.org/10.1007/s11432-024-4144-y}
  {\path{doi:10.1007/s11432-024-4144-y}}.

\bibitem{mecharbat2025moenas}
L.~A. Mecharbat, A.~Marchisio, M.~Shafique, M.~M. Ghassemi, T.~Alhanai, Moenas:
  Mixture-of-expert based neural architecture search for jointly accurate,
  fair, and robust edge deep neural networks, arXiv preprint arXiv:2502.07422
  (2025).
\newblock \href {https://doi.org/10.48550/arXiv.2502.07422}
  {\path{doi:10.48550/arXiv.2502.07422}}.

\bibitem{li2024theory}
H.~Li, S.~Lin, L.~Duan, Y.~Liang, N.~B. Shroff, Theory on mixture-of-experts in
  continual learning, arXiv preprint arXiv:2406.16437 (2024).
\newblock \href {https://doi.org/10.48550/arXiv.2406.16437}
  {\path{doi:10.48550/arXiv.2406.16437}}.

\bibitem{zhou2022mixture}
Y.~Zhou, T.~Lei, H.~Liu, N.~Du, Y.~Huang, V.~Zhao, A.~M. Dai, Q.~V. Le,
  J.~Laudon, et~al., Mixture-of-experts with expert choice routing, Advances in
  Neural Information Processing Systems 35 (2022) 7103--7114.
\newblock \href {https://doi.org/10.52202/068431-0515}
  {\path{doi:10.52202/068431-0515}}.

\bibitem{dempster1977em}
A.~P. Dempster, N.~M. Laird, D.~B. Rubin, Maximum likelihood from incomplete
  data via the {EM} algorithm, Journal of the Royal Statistical Society, Series
  B 39~(1) (1977) 1--38.
\newblock \href {https://doi.org/10.1111/j.2517-6161.1977.tb01600.x}
  {\path{doi:10.1111/j.2517-6161.1977.tb01600.x}}.

\bibitem{wu1983convergence}
C.~F.~J. Wu, On the convergence properties of the {EM} algorithm, The Annals of
  Statistics 11~(1) (1983) 95--103.
\newblock \href {https://doi.org/10.1214/aos/1176346060}
  {\path{doi:10.1214/aos/1176346060}}.

\bibitem{cappe2009online}
O.~Capp{\'e}, {\'E}.~Moulines, On-line expectation--maximization algorithm for
  latent data models, Journal of the Royal Statistical Society, Series B 71~(3)
  (2009) 593--613.
\newblock \href {https://doi.org/10.1111/j.1467-9868.2009.00698.x}
  {\path{doi:10.1111/j.1467-9868.2009.00698.x}}.

\bibitem{godahewa2021monash}
R.~Godahewa, C.~Bergmeir, G.~I. Webb, R.~J. Hyndman, P.~Montero-Manso, Monash
  time series forecasting archive, in: Proceedings of the Neural Information
  Processing Systems Track on Datasets and Benchmarks, 2021.
\newblock \href {https://doi.org/10.48550/arXiv.2105.06643}
  {\path{doi:10.48550/arXiv.2105.06643}}.

\bibitem{ni2024mole}
R.~Ni, Z.~Lin, S.~Wang, G.~Fanti, Mixture-of-linear-experts for long-term time
  series forecasting, in: International Conference on Artificial Intelligence
  and Statistics (AISTATS), 2024.
\newblock \href {https://doi.org/10.48550/arXiv.2312.06786}
  {\path{doi:10.48550/arXiv.2312.06786}}.

\bibitem{robbins1971convergence}
H.~Robbins, D.~Siegmund, A convergence theorem for non negative almost
  supermartingales and some applications, in: Optimizing Methods in Statistics,
  Academic Press, 1971, pp. 233--257.
\newblock \href {https://doi.org/10.1007/978-1-4612-5110-1_10}
  {\path{doi:10.1007/978-1-4612-5110-1_10}}.

\end{thebibliography}

\end{document}